\documentclass[twoside,11pt]{article}

\let\JMLRbibstyle\bibliographystyle
\renewcommand{\bibliographystyle}[1]{\JMLRbibstyle{abbrvnat-short}}
\usepackage[preprint,abbrvbib]{jmlr2e}
\let\bibliographystyle\JMLRbibstyle

\providecommand{\proofname}{Proof}
\renewenvironment{proof}[1][\proofname]
  {\par\noindent{\normalfont\bfseries #1\ }}
  {\hfill\BlackBox\\[2mm]}

\usepackage[utf8]{inputenc}

\usepackage{amsmath,verbatim,amssymb,amsfonts,amscd,graphicx}
\usepackage{xcolor,slashed}
\usepackage{graphics}
\usepackage{array}
\usepackage{hhline}
\usepackage{enumitem}

\usepackage{euscript}
\usepackage{url,hyperref}
\usepackage{mathrsfs}

\let\be\relax
\newcommand{\be}{\begin{equation}}
\newcommand{\ba}{\begin{aligned}}
\newcommand{\bee}{\begin{equation*}}
\newcommand{\ee}{\end{equation}}
\newcommand{\ea}{\end{aligned}}
\let\eee\relax
\newcommand{\eee}{\end{equation*}}
\newcommand{\bea}{\begin{equation} \begin{aligned} }
\newcommand{\eea}{\end{aligned}\end{equation} }

\newcommand{\R}{{ \mathbb{R}  }}

\newcommand{\Z}{{ \mathbb{Z}  }}
\newcommand{\Tn}{{ \mathbb{T}^n  }}

\newcommand{\cP}{{\mathcal{P}}}

\newcommand{\Tan}{\textrm{Tan}}

\DeclareMathOperator{\grad}{grad}

\usepackage{booktabs}
\usepackage{nicefrac}
\usepackage{microtype}
\usepackage{xurl}
\newif\ifcoauthorreview
\coauthorreviewfalse

\hypersetup{
  hidelinks,
  pdftitle={Local exponential stability of mean-field Langevin descent--ascent and associated particle system},
  pdfauthor={Geuntaek Seo, Minseop Shin, Pierre Monmarch\'e, Beomjun Choi},
  pdfkeywords={mean-field Langevin dynamics, descent-ascent, mixed Nash equilibrium, Wasserstein stability, interacting particle systems}
}

\ShortHeadings{Local exponential stability of MFL-DA}{Seo, Shin, Monmarch\'e and Choi}
\firstpageno{1}

\title{Local exponential stability of mean-field Langevin descent--ascent and associated particle system}

\author{%
\name Geuntaek Seo \email gtseo@postech.ac.kr \\
\addr Department of Mathematics, Pohang University of Science and Technology
\AND
\name Minseop Shin \email msshin03@kaist.ac.kr \\
\addr Department of Mathematical Sciences, Korea Advanced Institute of Science and Technology
\AND
\name Pierre Monmarch\'e \email pierre.monmarche@univ-eiffel.fr \\
\addr Laboratoire d'Analyse et de Math\'ematiques Appliqu\'ees, Universit\'e Gustave Eiffel
\AND
\name Beomjun Choi
\email bchoi@kaist.ac.kr \\
\addr Department of Mathematical Sciences, Korea Advanced Institute of Science and Technology 
}

\begin{document}

\maketitle
\thispagestyle{plain}

\begin{abstract}We study the mean-field Langevin descent--ascent (MFL-DA), a coupled optimization dynamics on the space of probability measures for entropically regularized two-player zero-sum games, together with its associated interacting particle system. For general nonconvex--nonconcave payoffs, Wang and Chizat (COLT 2024) asked whether the original single-timescale MFL-DA converges to the mixed Nash equilibrium and, if so, at what rate. We prove a local affirmative answer in Wasserstein space: if the initial datum is sufficiently close to the mixed Nash equilibrium, then the mean-field dynamics converges to it exponentially fast at a quantitative rate. We further show that the finite-$N$ particle system inherits this stability up to times exponential in $N$, with an $N$-independent exponential rate modulo a finite-particle error floor. Combined with the recent counterexample of Mourrat and Pillaud--Vivien for MFL-DA, which shows that global convergence cannot hold in general, our theorem completes the positive local counterpart of the Wang--Chizat question: the mixed Nash equilibrium has a robust basin of attraction, stable under both the mean-field flow and its finite-particle approximation.
\end{abstract}


\raggedbottom
\section{Introduction} \label{sec-intro}
Two-player zero-sum games or minimax problems arise throughout modern machine learning, most notably in generative adversarial networks (GAN) \citep{goodfellow2014generativeadversarialnetworks} and Wasserstein GAN \citep{arjovsky2017wassersteingan}. For a payoff function \(f(x,y)\), the simultaneous pure-strategy game seeks a saddle point, equivalently a pure-strategy Nash equilibrium, namely a pair \((x^*,y^*)\) such that
\[
f(x^*,y)\le f(x^*,y^*)\le f(x,y^*)
\qquad \text{for all } x,y.
\]
If such a point exists, its value agrees with the minimax/maximin value:
\[
f(x^*,y^*)=\min_x\max_y f(x,y)=\max_y\min_x f(x,y).
\]
In general, however, the minimax value problem alone does not imply the existence of a pure-strategy equilibrium, and the standard gradient descent--ascent (GDA) dynamics may suffer from non-existence of equilibria, degeneracy, and lack of convergence. To address these issues, \cite{domingoenrich2020} proposed a mean-field relaxation in which strategies are probability measures, i.e., mixed strategies, rather than points. From a game-theoretic perspective, this mean-field formulation can be interpreted
as a learning dynamics in continuous-population games, with entropic
regularization playing the role of bounded rationality; see, e.g., \citep{sandholm2010population,mertikopoulos2016learning}.

Let $\mathcal X=\mathcal Y=\Tn:=\R^n/\Z^n$\footnote{Following the setting of 
\cite{WC24_MFLDA}, we work with the flat torus state space. The conclusion
Section~\ref{sec-conclusion} discusses extensions to closed
Riemannian manifolds and the additional issues on $\R^n$.} and let the payoff function \(f:\mathcal X\times\mathcal Y\to\R\) be of class \(C^2\).
Consider the function $F:\mathcal P(\mathcal X)\times\mathcal P(\mathcal Y)\to[-\infty,\infty]$ defined by
\be \label{eq-functional}
F(\mu,\nu)=\iint f(x,y)\,d\mu(x)d\nu(y)+H(\mu)-H(\nu),
\ee
where $H(\mu)=\int \log\bigl(\frac{d\mu}{dx}\bigr)\,d\mu$ if $\mu\ll dx$ and $H(\mu)=\infty$ otherwise.
The entropic terms enforce absolute continuity and induce diffusive regularization in the associated dynamics.

We study the natural Wasserstein descent--ascent flow associated with $F$.
Identifying $\mu_t,\nu_t$ with their densities\footnote{If there is no confusion, we often identify a measure with its density with respect to the Lebesgue measure.}, the resulting evolution\footnote{See Section~\ref{sec-formalexplanation} for the Wasserstein gradient descent--ascent formulation.} reads
\bea\label{eq-main}
\begin{cases}\begin{aligned}
\partial_t \mu_t&=\Delta\mu_t+\nabla\cdot\Bigl(\mu_t\nabla_x\int_{\mathcal Y} f(x,y)\,d\nu_t(y)\Bigr),\\
\partial_t \nu_t&=\Delta\nu_t-\nabla\cdot\Bigl(\nu_t\nabla_y\int_{\mathcal X} f(x,y)\,d\mu_t(x)\Bigr).
\end{aligned}
\end{cases}
\eea
Following \cite{Kim2024SymmetricML,WC24_MFLDA}, we call \eqref{eq-main} the \emph{mean-field Langevin descent--ascent} (MFL-DA) equation, as it arises as the mean-field limit of the interacting particle system driven by Langevin descent--ascent dynamics in \eqref{eq:particle} \cite{domingoenrich2020,Kim2024SymmetricML}.

The mean-field counterpart of a pure-strategy Nash equilibrium is a \emph{mixed Nash equilibrium} (MNE), namely a pair $(\mu^*,\nu^*)\in\mathcal P(\mathcal X)\times\mathcal P(\mathcal Y)$ satisfying
\begin{equation}\label{eq-MNE}
F(\mu^*,\nu)\le F(\mu^*,\nu^*)\le F(\mu,\nu^*)\qquad\forall(\mu,\nu)\in\mathcal P(\mathcal X)\times\mathcal P(\mathcal Y).
\end{equation}
When such an equilibrium exists, its value agrees with the corresponding mixed-strategy minimax/maximin value.
An MNE is a fixed point of \eqref{eq-main} (and the converse is also true), and moreover for each $f\in C^2$ there exists a unique fixed point \cite[Theorem~4]{domingoenrich2020}.
Our first main result establishes that this unique equilibrium is locally exponentially stable in the Wasserstein space.

\begin{theorem}[Exponential stability of Nash equilibrium] \label{thm-2} For a given $f\in C^2$, let $(\mu^*,\nu^*)$ be the unique mixed Nash equilibrium of \eqref{eq-functional}. There exist $\delta>0$, $\lambda>0$, and $C<\infty$ depending on $f$ with the following significance: if the initial data satisfies $W_2(\mu_0,\mu^*)+W_2(\nu_0,\nu^*)<\delta$, then the MFL-DA \eqref{eq-main} running from $(\mu_0,\nu_0)$\footnote{MFL-DA \eqref{eq-main}  admits a unique solution for every initialization $(\mu_0,\nu_0) \in \mathcal{P}(\Tn)^2 $. The solution becomes regular immediately for $t>0$. See \cite[Prop. 1.1]{de2023reducing} or Theorem~\ref{thm:main_appendix}.} converges to $(\mu^*,\nu^*)$. More precisely, the densities converge in the $C^{1,\alpha}$-norm for every $\alpha\in(0,1)$, and 
\be \label{eq-maindecay}
W_2(\mu_t,\mu^*)+W_2(\nu_t,\nu^*)\le C e^{-\lambda t}\,\bigl(W_2(\mu_0,\mu^*)+W_2(\nu_0,\nu^*)\bigr).
\ee See Remark~\ref{remark-sharpasymp} for an explicit lower bound for the rate of convergence $\lambda>0$.
\end{theorem}

Previous works often write \eqref{eq-main} with diffusion coefficient
$\beta^{-1}$, where $\beta>0$ is the inverse temperature.
This normalization is equivalent to ours by a simple rescaling:
if $(\mu_t^\beta,\nu_t^\beta)$ solves the equation with diffusion
coefficient $\beta^{-1}$ and payoff $f$, then
\[
(\widetilde\mu_s,\widetilde\nu_s)
:=(\mu_{\beta s}^\beta,\nu_{\beta s}^\beta)
\]
solves \eqref{eq-main} with unit diffusion and payoff $\beta f$.
We therefore state our results with $\beta=1$ without loss of generality.

The analysis of Wasserstein gradient flows has seen major developments in recent years.
For single-player mean-field Langevin dynamics (obtained from \eqref{eq-main} by taking $f(x,y)=f(x)$ and ignoring $\nu_t$), global exponential convergence to equilibrium can easily be deduced from functional inequalities such as the log-Sobolev inequality (LSI), with constants depending on the potential and the domain.
Related convergence results are also available for particle-interacting models with additional pairwise interaction energies; for results in this direction and applications in machine learning, see \citep{MR2053570,CB18,Monmarche:2025aa,CJS25} and references therein.

For the descent--ascent dynamics \eqref{eq-main}, a comparable convergence theory has been incomplete.
Convergence is known under strong convex--concavity assumptions on \(f\), or for modified dynamics; see Section~\ref{sec-related}.
However, for general \(f\in C^2\), the convergence, or possible non-convergence, of the original single-timescale dynamics has remained largely unknown, as noted in \citep{WC24_MFLDA}.
Notably, \citet{pmlr-v202-lu23b} used a LSI to construct a Lyapunov function and proved convergence in a sufficiently separated two-timescale regime.
Since the two-timescale approach can be interpreted as quasi-statically freezing one of the variables, it is plausible that a LSI-type control alone does not directly resolve the original single-timescale problem.
This is consistent with the finite-dimensional examples of \citet{Mourrat26PL}, which show that even two-sided Polyak--{\L}ojasiewicz inequalities do not preclude cycling behavior for gradient descent--ascent dynamics.
This perspective partly motivated the questions posed in \citep{WC24_MFLDA}: progress in this model setting would sharpen our understanding of Wasserstein geometry for optimization.

Specifically, \cite{WC24_MFLDA} asked:
(i) whether the MFL-DA \eqref{eq-main} converges globally to \((\mu^*,\nu^*)\) in \(W_2\), \(\mathrm{KL}\)-divergence, or Nikaido--Isoda (NI) error;
(ii) if global convergence is too much to expect, whether the equilibrium is locally stable, i.e., whether convergence holds from a neighborhood of \((\mu^*,\nu^*)\) in one of these notions; and
(iii) in either case, what the quantitative rate is.
Theorem~\ref{thm-2} provides an affirmative answer to the local questions (ii) and (iii), with an explicit exponential rate as stated in Remark~\ref{remark-sharpasymp}. In the present setting, \(W_2^2\) can be controlled by the relative entropy \(\mathrm{KL}\), and \(\mathrm{KL}\) can in turn be controlled by the NI error; see \eqref{eq:W2-KL-NI}.
Moreover, convergence of densities in \(C^{0}\) implies convergence of the NI error to zero; see Lemma~\ref{lem:C1alpha-implies-NI} in Appendix~\ref{appendix-strong-NI}.
Consequently, Theorem~\ref{thm-2} yields convergence in stronger senses while requiring a weaker proximity assumption at initialization.

Indeed, the local stability in Theorem~\ref{thm-2} is essentially optimal at this level of generality: a recent counterexample of \citet{MourratPillaudVivien26Oscillating} shows that global convergence of the MFL-DA cannot hold in general: they construct smooth games with oscillatory, non-convergent solutions. One may wonder whether compactness changes this conclusion. Their argument carries over after a simple smooth periodization of the payoff, as briefly explained in Appendix~\ref{app:periodic-counterexample}. Thus the question (i) has a negative answer in general. Together with Theorem~\ref{thm-2}, this answers the questions raised by \citet{WC24_MFLDA} in a natural sense: global convergence is false in general, while the mixed Nash equilibrium is locally exponentially stable.


\medskip

In practice, the mean-field dynamics~\eqref{eq-main} is approximated by the corresponding $N$--particle system, which is the coupled diffusion on $(\Tn)^{2N}$:
\begin{equation}\label{eq:particle}
\begin{cases}
dX_t^i \;=\; - \frac1N \sum_{k=1}^N \nabla_x f (X_t^i,Y_t^k)\,dt + \sqrt{2}\,dB_t^i,\\[1mm]
dY_t^j \;=\; +\frac1N \sum_{k=1}^N \nabla _yf(X_t^k,Y_t^j)\,dt + \sqrt{2}\,dW_t^j,
\end{cases}
\qquad i,j=1,\dots,N,
\end{equation}
with independent Brownian motions $(B^i)_{i\le N},(W^j)_{j\le N}$. The fact that the empirical measures
\begin{equation}
    \label{eq:empirical_particle}
\mu_t^N:=\frac1N\sum_{i=1}^N\delta_{X_t^i},\qquad \nu_t^N:=\frac1N\sum_{j=1}^N\delta_{Y_t^j}
\end{equation}
converge  almost surely, for any $t>0$, as $N\rightarrow \infty$ to the solutions of \eqref{eq-main} provided this holds at $t=0$ is well known. However, the  error bounds to quantify this convergence are of order $e^{Lt} N^{-\eta}$ for some $L,\eta>0$ and are thus informative only up to times of order $\ln N$. Alternatively, for a fixed $N\in\mathbb N$, it is clear by standard Markov arguments that the system~\eqref{eq:particle} is ergodic and that its law converges exponentially fast as $t\rightarrow \infty$ to its invariant measure. However, the convergence rate obtained in this way would depend strongly on $N$ and would thus not be very informative in practice. By contrast, our next result provides the particle counterpart of Theorem~\ref{thm-2}  and involves a long-time (approximate) convergence rate independent of $N$.
\begin{theorem}[Particle stability over exponential times]
   \label{thm-particle} 
   There exist $\delta,\lambda,\theta,\eta>0$ and $C<\infty$ depending on $f$ with the following significance: for any $N \ge 1$, if the initial data of~\eqref{eq:particle} satisfies 
   \be \label{eq:initial_particle}
   \mathbb P \left(  W_2(\mu_0^N,\mu^*)+W_2(\nu_0^N,\nu^*)  \ge \delta\right) \le R_0 N^{-\eta}\ee 
   for some $R_0>0$,    then for all $t \leqslant e^{\theta N}$
\be \label{eq-maindecay-particle}
\mathbb E \left [ W_2(\mu_t^N,\mu^*)+W_2(\nu_t^N,\nu^*) \right] \le C e^{-\lambda t}\,\mathbb E \left[W_2(\mu_0^N,\mu^*)+W_2(\nu_0^N,\nu^*)\right] + C(1+R_0) N^{-\eta}.
\ee
\end{theorem}

The dimension dependence of $\eta$ and the optimality of the
error-floor exponent for $n>4$ are discussed in
Remark~\ref{rem:particle-sampling-rate}.

From \cite{fournier2015rate}, the condition~\eqref{eq:initial_particle} is met if the particles  are initially i.i.d. with distribution in a suitable Wasserstein ball around $\mu^*\otimes\nu^*$. The restriction to times smaller than $e^{\theta N}$ is not a problem in practice since this time-scale is never reached on actual simulations when $N$ is large. Moreover, this restriction is necessary in our analysis which relies only on local information around the MNE as Theorem~\ref{thm-2}, while by ergodicity the system~\eqref{eq:particle} will eventually leave any neighborhood of the MNE given sufficient time. In particular, if the mean-field equation~\eqref{eq-main} admits as in~\cite{MourratPillaudVivien26Oscillating} a stable periodic orbit, then it is unknown whether the proportion of time spent by the particle system in the vicinity of this orbit is negligible or not in the long-time, and thus whether a bound of the form~\eqref{eq-maindecay-particle} can hold for all times $t\ge 0$. The only analysis in a similar spirit to Theorem~\ref{thm-particle} (local convergence for mean-field particles over exponentially long time-scale) we are aware of is~\cite{monmarche2025long} for the mean-field Langevin process.

\section{Related work}\label{sec-related}

\paragraph{\textbf{Mean-field minimax dynamics.}}
Mean-field and Wasserstein gradient flow viewpoints have become central in the
analysis of learning dynamics on spaces of probability measures
\citep{CB18,MMN18,chizatmean}. These ideas have recently been extended from
minimization to stochastic minimax optimization and games
\citep{domingoenrich2020,WC24_MFLDA,pmlr-v291-cai25a,liu2025convergence},
building on classical McKean--Vlasov and interacting-particle theory
\citep{MR1842429,MR2053570}. In finite-dimensional zero-sum games, pure Nash
equilibria may fail to exist or be dynamically unstable beyond restrictive
settings \citep{mazumdar2019findinglocalnashequilibria,daskalakis2020complexityconstrainedminmaxoptimization}.
This motivates the mixed-strategy formulation, where players choose probability
measures. In this direction, \cite{pmlr-v97-hsieh19b} studied mixed equilibria in
GAN-type models, and \cite{domingoenrich2020} introduced the entropy-regularized
mean-field game \eqref{eq-functional}, proved uniqueness of the MNE, and derived
the MFL-DA equation \eqref{eq-main} as the mean-field limit of a stochastic
descent--ascent particle system.

\paragraph{\textbf{Modified dynamics and structured regimes.}}
Several convergence results are known for variants of \eqref{eq-main} or under
additional structure. In the quasistatic two-timescale limit, the fast variable
is replaced by an entropy-regularized best response, reducing the dynamics to a
single-player Wasserstein gradient flow; global convergence is proved in
\citep{ma2022provablyconvergentquasistaticdynamics}. For separated but finite
timescales, \cite{pmlr-v202-lu23b} prove global convergence under an LSI
assumption via a Lyapunov approach. Other modifications include temperature
annealing and time-averaged gradients: see \citep{pmlr-v202-lu23b,LuMon25} for
annealing schedules, and \citep{Kim2024SymmetricML,LuMon25} for averaged-gradient
dynamics. Langevin particle approximations and related discretized schemes have
also been studied in mean-field learning dynamics, with quantitative guarantees
often relying on convexity, LSI, or related structural assumptions
\citep{pmlr-v291-cai25a,Wang_2025,liu2025convergence,suzuki2023convergencemeanfieldlangevindynamics}.

\paragraph{\textbf{Single-timescale regime.}}
The original single-timescale MFL-DA \eqref{eq-main} is less understood for general \(f\). Existing convergence results typically rely on strong
convex--concavity of $f$ or related global contraction assumptions
\citep{pmlr-v291-cai25a,conger2025wgf,isobe2025}. The main novelty of the present
work is to prove local exponential stability near the mixed Nash equilibrium
without imposing global convex--concavity.

\paragraph{\textbf{Comparison with independent work.}}
Shortly after our first preprint appeared,\footnote{Our first
preprint contained Theorem~\ref{thm-2}, but not the particle
approximation result, Theorem~\ref{thm-particle}.}
\citet{WangChizat26Local} posted a preprint establishing closely related
local stability results for single-timescale MFL-DA. They note that their
result is part of Wang's PhD thesis and that they became aware of our work
while finalizing their preprint.
In light of this account and the methodological differences described below,
we regard the two works as independent.

In the present compact setting, they prove local exponential stability for
initial data sufficiently close to the MNE in \(\chi^2\) divergence, the
squared relative \(L^2\)-distance, whereas Theorem~\ref{thm-2} assumes only
\(W_2\)-smallness. Their estimate also implies \(W_2\) convergence; the
distinction is the topology of the initial perturbation. A \(W_2\)-neighborhood
contains singular and atomic measures, whose \(\chi^2\) divergence relative
to the smooth MNE is infinite. Both proofs use equilibrium coercivity:
the Wasserstein-Hessian gap, the spectral gap of the weighted Laplacian,
and the marginal Poincar\'e gaps are equivalent (Remark~\ref{remark-alternative})
and are automatic for the smooth, strictly positive equilibrium densities
on the torus. Their Eulerian Lyapunov functional combines \(\chi^2\) divergence
with weighted \(H^{-1}\) norms, which encode the linearized Wasserstein
geometry and capture the zero-sum cancellation. Our proof propagates the
equilibrium Hessian gap along nearby transport geodesics, obtaining local
displacement convexity--concavity and, through above-tangent inequalities,
a contraction toward the MNE in \(W_2\); see Section~\ref{sec-formalexplanation}.
Positive-time smoothing, interpolation, and a first-exit argument allow us
to apply this contraction to initial measures in the \(W_2\) basin.

This choice of topology is also essential for our proof of
Theorem~\ref{thm-particle}: on each time block, the comparison mean-field flow
is restarted from the current empirical measure, which is atomic but may be
arbitrarily close to the MNE in \(W_2\). There is a corresponding tradeoff.
\citet{WangChizat26Local} obtain a more explicit \(\chi^2\) basin and prefactor,
and their framework also accommodates noncompact Euclidean settings when
well-posedness, Poincar\'e inequalities, and the required derivative bounds
are verified. Our \(W_2\) basin is less explicit because of the smoothing
step, but it accommodates atomic data, supports the particle stability
result, and yields \(C^{1,\alpha}\) convergence. Extending
this approach to noncompact spaces would require additional weighted
smoothing, moment or tail control, and stability estimates for optimal
transport on noncompact spaces.

\section{Mean-field stability: proof of Theorem~\ref{thm-2}}
We first give an intuitive, high-level explanation of the main idea behind our analysis and sketch the proof strategy.
\subsection{Heuristic explanation of stability mechanism} \label{sec-formalexplanation}
\medskip
\noindent
\textbf{(1) Wasserstein GDA viewpoint.}
Our starting observation is that the mean-field learning dynamics
\eqref{eq-main} can be interpreted as a gradient
descent--ascent (GDA) flow of the functional $F(\mu,\nu)$ on the product
Wasserstein space $\mathcal P_2(\mathcal X)\times\mathcal P_2(\mathcal Y)$.
Indeed, writing the equations in continuity form $\mu_t + \nabla \cdot (\mu_t v_t)=0$ and  $\nu_t + \nabla \cdot ( \nu_t w_t)=0$ gives the velocities
\[
\begin{aligned}
v_t&=-\nabla\frac{\delta F}{\delta\mu}
=-\nabla\log\mu_t-\nabla_x\!\int f(\cdot,y)\,d\nu_t(y),\\
w_t&=+\nabla\frac{\delta F}{\delta\nu}
=-\nabla\log\nu_t+\nabla_y\!\int f(x,\cdot)\,d\mu_t(x).
\end{aligned}
\]
Near a mixed Nash equilibrium $(\mu^*,\nu^*)$, the local behavior of the dynamics
is governed by the second variation (Hessian) of $F$ in the sense
of Otto's formal Riemannian calculus.

\medskip
\noindent
\textbf{(2) Hessian and elliptic structure.}
Formally differentiating $\mu\mapsto F(\mu,\nu^*)$ along smooth Wasserstein
geodesics, one finds that the Hessian at $\mu^*$ induces a symmetric bilinear
form on the tangent space $\Tan_{\mu^*}\mathcal P_2(\mathcal X)$. Recall that the tangent space consists of gradient vector fields $\Phi=\nabla\phi$.
More precisely, by applying Lemma~\ref{lem:first-second-var} (variation formulae) for $\mu \mapsto  F(\mu,\nu^*)=:G(\mu)$, this Hessian reads
\[
\frac{\partial^2}{\partial t\,\partial s}\Big|_{t=s=0}
F\bigl((\mathrm{id}+t\Phi_1+s\Phi_2)_\#\mu^*,\nu^*\bigr)
= \langle L\Phi_1,\Phi_2\rangle_{L^2_{\mu^*}}=\langle \Phi_1,L \Phi_2\rangle_{L^2_{\mu^*}}.
\]
Here $L=L_{\mu^*,\nu^*}$ is the associated symmetric operator acting on gradient vector fields defined by
\begin{equation}\label{eq-L}
L\Phi
= -\rho_*^{-1}\nabla\cdot(\rho_*\nabla\Phi)
  + \nabla^2 V_{\nu^*}\cdot\Phi,
\qquad
V_{\nu^*}(x)=\int_{\mathcal Y} f(x,y)\,d\nu^*(y),
\end{equation}
where $\rho_*=\frac{d\mu^*}{dx}$ is the density of $\mu^*$. The first term of $L$ corresponds to the contribution from the entropy (acting like a weighted Laplacian), and the second captures the interaction geometry. This operator $L$ is uniformly elliptic on the space of gradient vector fields.

If $V_{\nu^*}$ is (strongly) $\lambda$-convex, then $\nabla^2 V_{\nu^*} \ge \lambda I$, which implies that $L$ is positive definite in the sense that
\begin{equation}
\label{eq-198}
\langle L\Phi,\Phi\rangle_{L^2_{\mu^*}}
\ge \lambda\|\Phi\|_{L^2_{\mu^*}}^2 \quad \text{for all } \Phi=\nabla\phi,
\end{equation}
which is precisely the infinitesimal manifestation of $\lambda$-displacement convexity of $\mu\mapsto F(\mu,\nu^*)$.\footnote{Our convexity conventions and the corresponding Hessian and tangent inequalities are collected in Appendix~\ref{app:disp-convexity}.} An analogous statement holds in the $\nu$-variable with the sign reversed. 

\medskip
\noindent
\textbf{(3) Local convex--concavity.}
If $F$ were globally $\lambda$-displacement convex--concave for some $\lambda>0$ (which happens, for instance, when $f(x,y)$ is $\lambda$-convex--concave), the resulting Wasserstein GDA would behave exactly like the classical GDA for strongly
convex--concave functions, yielding exponential convergence; see Lemma~\ref{lemma-stability-cc} for the corresponding global contraction argument.

In general, however, $f(x,y)$ need not be strongly convex--concave, so global coercivity fails. Our first motivating observation is that
\emph{local (infinitesimal)} convex--concavity nevertheless holds at the equilibrium
$(\mu^*,\nu^*)$ restricted to the tangent space of gradient fields. i.e., $F$ is $\lambda$-displacement convex--concave `at $(\mu^*,\nu^*)$' for $\lambda=0$. This is consistent with the finite-dimensional intuition that at a saddle
point $(x^*,y^*)$ of a $C^2$ function $f$ one has
$\nabla_{xx}^2 f(x^*,y^*)\ge0$ and $\nabla_{yy}^2 f(x^*,y^*)\le0$.

Crucially, the entropic regularization provides a decisive advantage over the classical setting.
While the Hessian might only be non-negative definite in a finite-dimensional degenerate game, the diffusion introduced by the entropy term ensures a coercive \emph{spectral gap} for the linearized operator corresponding to the Wasserstein Hessian. Through spectral theory, we show that this mechanism upgrades the marginal non-negativity to \emph{strict} coercivity on the relevant tangent spaces (Lemma~\ref{lemma-convconc}).
This coercivity relies on the Gibbs structure at the equilibrium and its
Poincar\'e gap, and is not an automatic consequence of adding entropy
(Remark~\ref{rem:entropy-not-automatic}).
We further prove that this strict positivity is robust, extending to a neighborhood of the equilibrium (Lemma~\ref{lemma-locconvconc}).
Combined with the smoothing effect of \eqref{eq-main}, this restored local
convex--concave structure drives the exponential stability result in
Theorem~\ref{thm-2}.

\subsection{From coercivity to local exponential stability}\label{sec-MFLDA}
The proof of Theorem~\ref{thm-2} will be given at the end of this section. The next two lemmas constitute the most crucial part of our result. Recall that, for given absolutely continuous measure $\mu$, constant speed geodesics in Wasserstein space starting from $\mu$ can be locally written as a form $t\mapsto (id+ t\nabla \phi )_\# \mu$ for potential function $\phi$ (see Appendix~\ref{sec:ot} or \cite{mccann1997convexity}). We estimate the second derivative of $F$ along such geodesics, corresponding to the Hessian of $F$ in Wasserstein space. The analytic and transport estimates used below are collected in Appendix~\ref{app-prelim}.

\begin{lemma} [Spectral gap at Nash equilibrium]\label{lemma-convconc} Let $(\mu^*,\nu^*)$ be the unique MNE. There exists $\lambda_{gap}=\lambda_{gap}(f)>0$ such that, for any vector fields $\nabla \phi$ and $\nabla \psi \in W^{2,p}(\Tn)$ with $p>n$, there hold
\begin{align} \label{eq-conv1}\frac{d^2}{dt^2}\Big \vert_{t=0}F((id+t \nabla \phi)_\# \mu^* , \nu^*) &\ge +\lambda_{gap} \int |\nabla \phi |^2 d\mu^* , \\  \label{eq-conc1}
\frac{d^2}{dt^2}\Big \vert_{t=0} F(\mu^* , (id+t \nabla \psi)_\# \nu^*) &\le -\lambda_{gap} \int |\nabla \psi |^2 d\nu^* .\end{align}

\end{lemma} 
\begin{proof}
 We prove \eqref{eq-conv1}. A proof of \eqref{eq-conc1} follows by considering $\hat  F (\nu,\mu):= -F(\mu,\nu)$ and repeat the same proof of \eqref{eq-conv1}; see the last sentence of this proof. Let us define $V_{\nu^*}(x):= \int_{\mathcal{Y}} f(x,y) d\nu^*(y)$ and 
\[G( \mu  ) := F(\mu  , \nu^*)= H(\mu)+\int V_{\nu^*}(x) d\mu(x) - H(\nu^*).\]
Let us denote $\frac{d \mu^*}{dx}=:\rho_*$. It is known that  $\rho_*  \propto e^{-V_{\nu^*} }$  \citep{risken1989fpe}. More precisely, $\rho_*= e^{-V_{\nu^*}} / \int e^{-V_{\nu^*}}dx$. Similarly, $\frac{d \nu^*}{dy}$ satisfies $\frac{d \nu^*}{dy} =  e^{+U_{\mu^*}}/ \int e^{+U_{\mu^*}}dy$ for $U_{\mu^*}(y):= \int _{\mathcal{X}}f(x,y) d\mu^*(x) $.  We also note that the MNE $(\mu^*,\nu^*)$ has $C^2$ densities:  \begin{equation}\label{eq-C^2bound} C^{-1}\le \frac{d\mu^*}{dx}, \, \frac{d\nu^*}{dy} \le C, \quad \Big \Vert \frac{d\mu^*}{dx}\Big \Vert_{C^2},\Big \Vert \frac{d\nu^*}{dy}\Big \Vert_{C^2} \le C, \quad \text{ for some }C=C(\Vert f\Vert_{C^2})<\infty,   
\end{equation}
which follows from the observation $\Vert U_{\mu^*}\Vert_{C^2(\mathcal{Y})}$, $\Vert V_{\nu^*}\Vert_{C^2(\mathcal{X})}\le\Vert f\Vert_{C^2(\mathcal{X}\times \mathcal{Y})}$.

For $\nabla \phi \in W^{2,p}$ for $p>n$, we have the second variation formula 
\begin{equation} \label{eq-secondvariation}\frac{d^2}{d t ^2}\Big\vert_{t=0} G((id+t\nabla \phi )_\# \mu^* )= \int  (|\nabla^2  \phi|^2 + \nabla^2 V_{\nu^*}[\nabla\phi,\nabla\phi]) \, \rho_* dx . \end{equation}Here $|\nabla ^2 \phi|^2= \sum_{i,j} \nabla^2_{ij}\phi \nabla^2_{ij}\phi $ and $\nabla^2 V_{\nu^*}[\nabla\phi,\nabla\phi] = \sum_{i,j} \nabla^2_{ij}V_{\nu^*} \nabla _i\phi \nabla_j\phi  = \langle \nabla^2 V_{\nu^*} \nabla \phi ,\nabla \phi \rangle $. This formula has been well-known in the literature: e.g., formal derivations in  \cite[Formula 15.7]{villani2008optimal}, \cite[Section 3.1]{MR2053570}. Here the condition $\nabla \phi \in W^{2,p}$ is assumed for a rigorous justification shown in Lemma~\ref{lem:first-second-var}. Since $\rho_* \propto e^{-V_{\nu^*}}$, using regularities of $\rho_*$ and $V_{\nu^*}$,  \be \label{eq-102} \nabla^2 V_{\nu^*} = \nabla (-\rho_*  ^{-1} \nabla \rho_*  )= \rho_*^{-2} \nabla\rho_* \otimes \nabla\rho_* -\rho_*^{-1} \nabla^2 \rho_* .\ee 
After several applications of integration by parts using \eqref{eq-102}, we have 
\be \label{eq-101}\int  (|\nabla^2  \phi|^2 + \nabla^2 V_{\nu^*}[\nabla\phi,\nabla\phi]) \, \rho_* dx= \int |\Delta \phi -\langle \nabla\phi,\nabla V_{\nu^*}\rangle|^2 \,\rho_*dx.\ee 
See Lemma~\ref{lemma-A7eq101} where we derive \eqref{eq-101}. Later the authors learned that this observation is already well-known in the literature; e.g., \cite[Ch 1.16.1]{bakry2013analysis}.

By \eqref{eq-101}, the required coercivity also follows from the
Poincar\'e inequality for $\mu^*$; see Remark~\ref{remark-alternative}
for this alternative argument and quantitative bounds. We give a
self-contained proof by a compactness argument below.

We aim to prove a stronger statement that the constant \begin{equation} \label{eq-lambdagap}\lambda_{gap}:=\inf_{\phi \in W^{2,2}, \nabla \phi \neq 0 } \frac{\int  (|\nabla^2  \phi|^2 + \nabla^2 V_{\nu^*}[\nabla\phi,\nabla\phi]) \, \rho_* dx}{\int   |\nabla \phi|^2 \,\rho_*dx } \end{equation} is {positive} and the infimum is achieved by some $\phi_0 \in W^{2,2}$. The result \eqref{eq-conv1} follows since $W^{3,p}\subset W^{2,2}$ for $p>n$. From its definition and \eqref{eq-101}, $\lambda_{gap}$ is finite non-negative number and  there exists a sequence $\{\phi_i\}_{i=1}^\infty\subset W^{2,2}$ such that the quotient on the right hand side of \eqref{eq-lambdagap} converges to $\lambda_{gap}$. For such a sequence we may further assume $\int |\nabla \phi_i|^2 \rho_*dx =1$ and $\int \phi_i dx  =0 $ since replacing $\phi_i$ by $\bar\phi_i= \beta_i (\phi_i-\alpha_i)$, for $\alpha_i=\int \phi_i dx $, $\beta_i^{-2}=  \int |\nabla \phi_i|^2 \rho_*dx $,  does not change the quotient.

From the previous observations and the choice $\phi_i$, the numerator of quotient is bounded uniformly in $i$. In particular, this implies $\int |\nabla ^2 \phi_i |^2 dx \le C<\infty$.
By weak compactness and the compact embedding $W^{2,2}\Subset W^{1,2}$ in \eqref{eq:compact-embeddings}, there exists a subsequence $\{\phi_{n_j}\}\subset W^{2,2}$ and $\phi_0 \in W^{2,2}$ such that $\phi_{n_j}$ converges to $\phi_0$ weakly in $W^{2,2}$ and strongly in $W^{1,2}$. The strong convergence implies $  \int |\nabla \phi_0|^2 \,\rho_*= \lim_{j\to \infty} \int |\nabla \phi_{n_j}|^2 \, \rho_*=1 $. Due to lower-semi-continuity of $L^2$-norm under $L^2$ weak convergence, 
\[\int  (|\nabla^2  \phi_0|^2 + \nabla^2 V_{\nu^*}[\nabla\phi_0,\nabla\phi_0]) \, \rho_*  \le \liminf _{j\to \infty} \int   (|\nabla^2  \phi_{n_j}|^2 + \nabla^2 V_{\nu^*}[\nabla\phi_{n_j},\nabla\phi_{n_j}]) \, \rho_*  =\lambda_{gap}.\]
By the definition of $\lambda_{gap}$, the previous inequality has to be an equality. This shows that $\lambda_{gap}\ge0$ is achieved by some $\phi_0\in W^{2,2}$. It remains to prove $\lambda_{gap}>0$. If $\lambda_{gap}=0$, \eqref{eq-101} implies  
\[\Delta \phi_0 - \langle \nabla \phi_0 ,\nabla V_{\nu^*}\rangle =0,\]
which implies $\nabla \cdot (e^{-V_{\nu^*}} \nabla \phi_0)=0$. By multiplying with $\phi_0$ and then integrating by parts,
\begin{equation} \label{eq-energyargument} 0 = \int \phi_0 \nabla \cdot (e^{-V_{\nu^*}} \nabla \phi_0)=-\int e^{-V_{\nu^*}} | \nabla \phi_0 |^2 . \end{equation}
This contradicts $\Vert \nabla\phi_0 \Vert_{L^2_{\mu^*}}=1$ and thereby finishes \eqref{eq-conv1}.

For \eqref{eq-conc1}, following similarly as suggested in the beginning, we obtain $\bar \lambda_{gap}>0$ so that \eqref{eq-conc1} holds with equality for some $\bar \phi_0\in W^{2,2}$. We choose $\lambda_{gap}$ as the minimum of two resulting $\lambda_{gap}$ s obtained. 
\end{proof}

To turn Lemma~\ref{lemma-convconc} into a neighborhood statement, we use the continuous dependence of the second variation of $F$ on $(\mu,\nu)$. As a result, the convex--concavity constants persist, up to an arbitrarily small loss. Here and below, \(T_{\mu_0\to\mu_1}\) denotes the optimal transport map from \(\mu_0\) to \(\mu_1\).
\begin{lemma}[Local convex--concavity] \label{lemma-locconvconc}For every small $\eta >0$ and $p>n $, there exists $\delta>0$ such that if $\mu = (T_{\mu^*\to \mu})_\# \mu^*  $, $\nu = (T_{\nu^*\to \nu})_\# \nu^*$ with $\Vert T_{\mu^*\to \mu} -id \Vert_{W^{2,p}}$, $\Vert T_{\nu^*\to \nu} -id \Vert_{W^{2,p}}\le \delta$, then for any $\nabla \phi$ and $\nabla \psi \in W^{2,p}({\Tn})$, there hold
\begin{align}\label{eq-conv2} \frac{d^2}{dt^2}\Big \vert_{t=0}F((id+t \nabla \phi)_\# \mu , \nu) \ge +(1-\eta)\lambda_{gap} \int |\nabla \phi |^2 d\mu ,\\
    \label{eq-conc2} \frac{d^2}{dt^2}\Big \vert_{t=0} F(\mu , (id+t \nabla \psi)_\# \nu) \le -(1-\eta)\lambda_{gap} \int |\nabla \psi |^2 d\nu .\end{align}  
Here, $\lambda_{gap}=\lambda_{gap}(f)>0$ is the constant from Lemma~\ref{lemma-convconc} and $\delta $ depends on $f$, $p$ and $\eta$.
\end{lemma}

\begin{proof}
 Let us prove \eqref{eq-conv2} as \eqref{eq-conc2} similarly follows. 
    Recall we have $0<C^{-1}\le \frac {d\mu^*}{dx},\frac{d\nu^*}{dx} \le C$ for some $C=C(f)<\infty$ on $\Tn$. Observe that, for every $\xi >0$, there exists $\delta>0$ such that $\mu$ and $\nu$ in the statement satisfies $\Vert \mu-\mu^*\Vert _{W^{1,p}}+\Vert \nu-\nu^*\Vert _{W^{1,p}}\le \xi $ (see Lemma~\ref{lemma:sobolev_stability}). Using Morrey embedding $W^{1,p} \hookrightarrow C^{0,\alpha}$, after possibly replacing $\delta>0$ by smaller one, we may get $\Vert \mu-\mu^*\Vert _{C^0}+\Vert \nu-\nu^*\Vert _{C^0}\le \xi $. Let us prove the statement. Recall \eqref{eq-secondvariation} (Lemma~\ref{lem:first-second-var}) that, for $V_\nu(x) = \int _{\mathcal{Y}}f(x,y) \,d\nu(y)$, there holds
\[ \frac{d^2}{dt^2}\Big \vert_{t=0}F((id+t \nabla \phi)_\# \mu , \nu) = \int |\nabla^2  \phi|^2 + \nabla ^2 V_{\nu}[\nabla \phi ,\nabla \phi] \, d\mu .\]
Let us denote $\frac{d\mu}{dx}=\rho$. For a fixed small $\xi>0$,
    \bea & \int ( |\nabla^2  \phi|^2 + \nabla ^2 V_{\nu}[\nabla \phi ,\nabla \phi] )\rho = (1-\xi  )  \int( |\nabla^2  \phi|^2 + \nabla ^2 V_{\nu^*}[\nabla \phi ,\nabla \phi] )\rho_* \\
      &+ \int (\rho -(1-\xi )\rho_*) |\nabla^2 \phi|^2 + (\rho \nabla^2 V_{\nu}-(1-\xi )\rho_*\nabla^2 V_{\nu^*} )[\nabla \phi,\nabla \phi ]  \eea 
We rearrange the matrices in the last term 
      \bea \rho  \nabla^2 V_{\nu}-(1-\xi )\rho_*\nabla^2 V_{\nu^*} = \rho (\nabla^2 V_{\nu}-\nabla^2 V_{\nu^*})+ (\rho-\rho_*) \nabla^2 V_{\nu^*} +\xi  \rho_* \nabla^2 V_{\nu^*} .\eea 
      Note that $\nabla^2 V_{\nu}(x) =\int \nabla^2_xf(x,y)d\nu(y) $ and $\nabla^2 V_{\nu}\in C^0 $. Using $f\in C^2$, we conclude that there exists $C=C(f)<\infty$ such that for given $\xi>0$ there exists $\delta>0$ small such that, for all $x\in \Tn$,
      \[ \Big \Vert \frac { \rho(x)  \nabla^2 V_{\nu}(x)-(1-\xi )\rho_*(x)\nabla^2 V_{\nu^*}(x) }{\rho_*(x)}\Big \Vert 
 \le C\xi ,\]  and $ \rho(x) -(1-\xi )\rho_*(x) \ge 0 $. Here $\Vert \cdot \Vert $ is the operator norm of matrix. Combining above,
\bea\label{eq-1677}\int ( |\nabla^2  \phi|^2 + \nabla ^2 V_{\nu}[\nabla \phi ,\nabla \phi] )\rho \ge ((1-\xi) \lambda_{gap} - C \xi) \int |\nabla\phi|^2 \rho _*    .\eea
For $\eta>0$, choose $\xi>0$ and $\delta>0$ so that $(1-\xi) \lambda_{gap} - C \xi \ge( 1-\frac1 2 \eta)\lambda_{gap} $ and $\Vert 1-\frac{\rho_*}{\rho}\Vert_{L^\infty} \le \frac12 \eta $. 
\[\int ( |\nabla^2  \phi|^2 + \nabla ^2 V_{\nu}[\nabla \phi ,\nabla \phi] )\rho  \ge (1-\frac12 \eta)(1-\frac 12 \eta )\lambda_{gap} \int |\nabla \phi |^2 \rho , \]
and the lemma follows as $(1-\frac12 \eta)^2 \ge (1-\eta)$.
\end{proof}

Next, we explain how the local coercivity obtained above enters the dynamics. The guiding model example is the finite-dimensional strongly convex--concave GDA
\[
\dot x_t=-\nabla_x f(x_t,y_t),\qquad
\dot y_t=\nabla_y f(x_t,y_t).
\]
If \(f\) is \(\lambda\)-strongly convex in \(x\), \(\lambda\)-strongly concave
in \(y\), and  \((x^*,y^*)\) is a saddle point, then
\[
\begin{aligned}
&\frac12\frac{d}{dt}\bigl(|x_t-x^*|^2+|y_t-y^*|^2\bigr)
=
-\langle \nabla_x f(x_t,y_t),x_t-x^*\rangle
+\langle \nabla_y f(x_t,y_t),y_t-y^*\rangle  \\
&\le
-\tfrac\lambda 2\bigl(|x_t-x^*|^2+|y_t-y^*|^2\bigr)
+ f(x^*,y_t)-f(x_t,y^*) \le
-\tfrac\lambda 2\bigl(|x_t-x^*|^2+|y_t-y^*|^2\bigr),
\end{aligned}
\]
where the first inequality uses strong convexity--concavity (above tangent inequality) and the last inequality uses the saddle-point property.\footnote{Applying the same strong convexity--concavity estimates at \((x^*,y^*)\) gives the sharper bound \(\frac12E'(t)\le-\lambda E(t)\), where \(E(t)=|x_t-x^*|^2+|y_t-y^*|^2\).} Thus the squared distance to the equilibrium decays exponentially. The Wasserstein argument below is the same computation with Euclidean convexity replaced by displacement convexity (see Appendix~\ref{app:disp-convexity}) and with the metric derivative encoded through the evolution variational inequality (EVI). 

We first record the corresponding global estimate in the Wasserstein setting: a rigorous proof is given in Appendix~\ref{App-CCstability}. It is useful even when the convexity constant is negative: in that case it gives an a priori exponential growth bound, while for a positive constant it gives global exponential convergence.
\begin{lemma}[Stability]\label{lemma-stability}
Suppose that $\nabla^2_{xx}f(x,y)\ge \lambda' I$, $
\nabla^2_{yy}f(x,y)\le -\lambda' I$ for all \(x,y\in\mathbb T^n\), for some \(\lambda'>-\infty\). Then, for any
initial data \((\mu_0,\nu_0)\), the MFL-DA satisfies
\[
W_2^2(\mu_t,\mu^*)+W_2^2(\nu_t,\nu^*)
\le
e^{-2\lambda' t}
\bigl(W_2^2(\mu_0,\mu^*)+W_2^2(\nu_0,\nu^*)\bigr).
\]
In particular, if \(\lambda'>0\), the equilibrium is globally exponentially
stable. Since \(f\in C^2(\mathbb T^n\times\mathbb T^n)\), such a constant
\(\lambda'=\lambda'(f)>-\infty\) always exists.
\begin{proof} The assumptions imply that \(\mu\mapsto F(\mu,\nu)\) is \(\lambda'\)-displacement convex for each fixed \(\nu\), and that
\(\nu\mapsto F(\mu,\nu)\) is \(\lambda'\)-displacement concave for each fixed \(\mu\) \citep{MR2053570}. Then the result follows from the more general statement proved in Lemma~\ref{lemma-stability-cc}.
\end{proof}
\end{lemma}
 The main contraction estimate below is a localized version of the same argument. Lemmas~\ref{lemma-convconc} and \ref{lemma-locconvconc} show that the functional has an effectively strong displacement convex--concave structure near equilibrium, yielding desired contraction as long as the solution remains close to the equilibrium.

\begin{proposition}[Local contraction]\label{prop-main}
Fix \(p>n\) and \(\lambda\in(0,\lambda_{gap})\).
There exists \(\varepsilon=\varepsilon(f,p,\lambda)>0\) such that the following holds.
Let \((\mu_t,\nu_t)\) be a solution on a time interval \([a,b]\), with \(0\le a<b\).
Assume that, for all \(t\in[a,b]\),
\begin{equation}\label{eq-prop-main}
\|\mu_t-\mu^*\|_{W^{1,p}}+\|\nu_t-\nu^*\|_{W^{1,p}}\le \varepsilon .
\end{equation}
Then, for a.e. \(t\in[a,b]\),
\begin{equation}\label{eq-lyapunov'}
\frac{d}{dt}\Bigl(W_2^2(\mu_t,\mu^*)+W_2^2(\nu_t,\nu^*)\Bigr)
\le
-2\lambda\Bigl(W_2^2(\mu_t,\mu^*)+W_2^2(\nu_t,\nu^*)\Bigr).
\end{equation}
\end{proposition}
\begin{proof}
Fix $p>n$ and $\lambda\in(0,\lambda_{gap})$. Choose $\eta\in(0,1)$ such that $\lambda \le (1-\eta)\lambda_{gap}$. Let $\delta=\delta(f,p,\eta)$ be as in Lemma~\ref{lemma-locconvconc} (local convex--concavity).

\smallskip
\noindent{Step 0. Control on the regularity and smallness of optimal transport map.}

\noindent By Lemma~\ref{lemma-IFT2}, taking $\varepsilon>0$ small enough we make sure \eqref{eq-prop-main} implies: (i) $\mu_t,\nu_t$ remain uniformly bounded away from $0$ and $\infty$ on $\Tn$, and (ii) the optimal maps $T_{\mu^*\to\mu_t}$ and $T_{\nu^*\to\nu_t}$ satisfy
\[
\|T_{\mu^*\to\mu_t}-id\|_{W^{2,p}}+\|T_{\nu^*\to\nu_t}-id\|_{W^{2,p}} \le \delta.
\]
Thus, for every $s\in[0,1]$, the displacement interpolations (constant speed geodesics)
\begin{equation}\label{eq-1066}
\mu_{t,s}:=\bigl(id+s(T_{\mu^*\to\mu_t}-id)\bigr)_\#\mu^*,
\qquad
\nu_{t,s}:=\bigl(id+s(T_{\nu^*\to\nu_t}-id)\bigr)_\#\nu^*
\end{equation}
also stay in the $\delta$-neighborhood required by Lemma~\ref{lemma-locconvconc}.

\smallskip

\noindent{Step 1. A local above-tangent inequality for the $\mu$-part.}

\noindent Fix $t\in[a,b]$.
Let $T_{\mu_t\to\mu^*}=id+\nabla\tilde\phi_t$ be the optimal transport map from $\mu_t$ to $\mu^*$.
Define a constant speed geodesic from $\mu_t$ to $\mu^*$ by
\[
\tilde\mu_{t,s}:=(id+s\nabla\tilde\phi_t)_\#\mu_t,\qquad s\in[0,1],
\]
so that $\tilde\mu_{t,0}=\mu_t$ and $\tilde\mu_{t,1}=\mu^*$. Using $T_{\mu_t \to \mu^*} = (T_{\mu^*\to \mu_t})^{-1}$ and \eqref{eq-1066}, we observe that  $\mu_{t,s}=\tilde \mu_{t,1-s}$ and similarly $\nu_{t,s}=\tilde \nu_{t,1-s}$.  Let
\[
\ell(s):=F(\tilde\mu_{t,s},\nu_t).
\]
In view of Lemma~\ref{lem:first-second-var} in Appendix~\ref{appendix-variationformula}, it is straightforward to see $s\mapsto \ell(s)$ is $C^2$ and thus Taylor's theorem gives
\[
\ell(1)=\ell(0)+\ell'(0)+\int_0^1 (1-s)\ell''(s)\,ds.
\]
A direct first-variation computation (Lemma~\ref{lem:first-second-var}) yields
\begin{equation}\label{eq-taylorexpandell}
\ell'(0)
=\Big\langle \nabla\Big(\log\frac{d\mu_t}{dx}+V_{\nu_t}\Big),\,\nabla\tilde\phi_t\Big\rangle_{L^2_{\mu_t}},
\qquad
V_{\nu_t}(x):=\int f(x,y)\,d\nu_t(y).
\ee
Next, an estimate on $\ell''(s)$ follows by the second variation along geodesic (Lemma~\ref{lemma-locconvconc}): \[\ell''(s)\ge (1-\eta)\lambda_{gap}  W_2^2(\mu_{t},\mu^*)\ge \lambda W_2^2( \mu_t,\mu^*).\] Explaining this derivation more in detail,  for each fixed $s'$, observe\[id +s \nabla  \tilde \phi_t =(id +(s-s') \nabla  \tilde \phi_t\circ (id +s'\nabla \tilde \phi_t)^{-1})\circ (id +s'\nabla \tilde \phi_t) .\] By a direct calculation, 
\[ \nabla  \tilde \phi_t\circ (id +s'\nabla \tilde \phi_t)^{-1}=\nabla \tilde \phi_{t,s'},\] where $\tilde \phi_{t,s'}(y):= \tilde \phi_t(x)+\frac{s'}{2} |\nabla \tilde \phi_t (x)|^2$, $x= (id +s'\nabla \tilde \phi_t)^{-1}(y)$. Thus Lemma~\ref{lemma-locconvconc} implies
\[\ell''(s)\ge (1-\eta)\lambda_{gap}  \int |\nabla \tilde \phi_{t,s}|^2 d \tilde \mu_{t,s}=(1-\eta)\lambda_{gap}\int |\nabla \tilde \phi_t|^2 d \mu_t = (1-\eta)\lambda_{gap} W_2^2(\mu_t,\mu^*).\]
Combining estimates above, from \eqref{eq-taylorexpandell} we infer 
\begin{equation}\label{eq:local-AT-mu}
F(\mu^*,\nu_t)-F(\mu_t,\nu_t)
\ge
\Big\langle \nabla\Big(\log\frac{d\mu_t}{dx}+V_{\nu_t}\Big),\,\nabla\tilde\phi_t\Big\rangle_{L^2_{\mu_t}}
+\frac{\lambda}{2}W_2^2(\mu_t,\mu^*).
\end{equation}

\smallskip

\noindent {Step 2. EVI for the $\mu$-part.}

\noindent Under MFL-DA, $\mu_t$ satisfies the continuity equation with velocity
\[
u_t=-\nabla\Big(\log\frac{d\mu_t}{dx}+V_{\nu_t}\Big).
\]
By the first variation formula \cite[Thm.~8.4.7]{AGS08} or  {Theorem~\ref{thm:first_variation_W2}}, for a.e.\ $t$,
\[
\frac{d}{dt}\frac12 W_2^2(\mu_t,\mu^*)
=
\big\langle -u_t,\,T_{\mu_t\to\mu^*}-id\big\rangle_{L^2_{\mu_t}}
=
\Big\langle \nabla\Big(\log\frac{d\mu_t}{dx}+V_{\nu_t}\Big),\,\nabla\tilde\phi_t\Big\rangle_{L^2_{\mu_t}}.
\]
Combining with \eqref{eq:local-AT-mu} yields the Evolution Variational Inequality(EVI) with respect to the equilibrium.
\begin{equation}\label{eq:local-EDI-mu}
\frac{d}{dt}\frac12 W_2^2(\mu_t,\mu^*)
\le
F(\mu^*,\nu_t)-F(\mu_t,\nu_t)-\frac{\lambda}{2}W_2^2(\mu_t,\mu^*).
\end{equation}

\smallskip

\noindent{Step 3. EVI for the $\nu$-part.}

\noindent The same argument (using the local concavity bound in Lemma~\ref{lemma-locconvconc})
gives for a.e.\ $t$,
\begin{equation}\label{eq:local-EDI-nu}
\frac{d}{dt}\frac12 W_2^2(\nu_t,\nu^*)
\le
F(\mu_t,\nu_t)-F(\mu_t,\nu^*)-\frac{\lambda}{2}W_2^2(\nu_t,\nu^*).
\end{equation}

\smallskip
\noindent{Step 4. Lyapunov inequality and improved cancellation.}

\noindent Set $\mathcal E(t):=W_2^2(\mu_t,\mu^*)+W_2^2(\nu_t,\nu^*)$.
Summing \eqref{eq:local-EDI-mu}--\eqref{eq:local-EDI-nu} yields
\[
\mathcal E'(t)
\le
-\lambda \mathcal E(t)+2\big(F(\mu^*,\nu_t)-F(\mu_t,\nu^*)\big).
\]
Finally, applying the same local above-tangent argument with $(\mu,\nu)$ near $(\mu^*,\nu^*)$
and using stationarity $\grad_{W_2,\mu}F(\mu^*,\nu^*)=\grad_{W_2,\nu}F(\mu^*,\nu^*)=0$, we get
\[
F(\mu_t,\nu^*)-F(\mu^*,\nu^*)\ge \frac{\lambda}{2}W_2^2(\mu_t,\mu^*),
\qquad
F(\mu^*,\nu_t)-F(\mu^*,\nu^*)\le -\frac{\lambda}{2}W_2^2(\nu_t,\nu^*).
\]
Hence
\[
F(\mu^*,\nu_t)-F(\mu_t,\nu^*)
\le -\frac{\lambda}{2}\mathcal E(t).
\]
Plugging back gives $\mathcal E'(t)\le -2\lambda \mathcal E(t)$ for a.e.\ $t\in[a,b]$,
which is exactly \eqref{eq-lyapunov'}.
\end{proof}

To conclude Theorem~\ref{thm-2} from Proposition~\ref{prop-main} (local contraction), we must address two technical points:
(i) we need to relate smallness in $W_2$-distance to the MNE to smallness in a $W^{1,p}$-type distance, since Proposition~\ref{prop-main} is formulated under the latter assumption; and
(ii) we must show that the required $W^{1,p}$-smallness persists (or can be recovered) along the dynamics.

We settle these issues by combining Lemma~\ref{lemma-stability} with the next two lemmas, which are standard in the theory of diffusion equations and Wasserstein geometry. 

\begin{lemma}[Smoothing] \label{lemma-smoothing}
For any $\mu_0$, $\nu_0$ in $\mathcal{P}(\Tn)$, $\delta>0$ and $\alpha\in(0,1)$, there is $C=C(f,\delta,\alpha)<\infty$ such that 
\[\Vert \mu_t\Vert_{C^{1,\alpha}}+\Vert \nu_t \Vert _{C^{1,\alpha}}\le C\quad \text{ for } t\ge \delta . \]

\end{lemma}

\begin{lemma}[Interpolation]\label{lemma-interpolation}
Let $0<\beta<\alpha<1$ and let $\rho_0\,dx,\rho_1\,dx \in \mathcal{P}(\Tn)$.
Assume $\|\rho_0\|_{C^{1,\alpha}}$, $\|\rho_1\|_{C^{1,\alpha}}\le C<\infty$.
Then for every $\varepsilon>0$ there exists $\delta=\delta(\varepsilon,C,\alpha,\beta)>0$
such that
\[
W_2(\rho_0,\rho_1)\le \delta
\quad\Longrightarrow\quad
\|\rho_0-\rho_1\|_{C^{1,\beta}}\le \varepsilon.
\]
{\begin{proof}[Proofs of Lemmas~\ref{lemma-smoothing} and \ref{lemma-interpolation}] 
{{See Appendix~\ref{app-MFLDA}.}}
\end{proof}
}
\end{lemma}
\begin{proof}[Proof of Theorem~\ref{thm-2}]
Fix $p>n$ and $\lambda\in(0,\lambda_{gap})$, let $\varepsilon_0$ be the
constant in Proposition~\ref{prop-main}, and set
$E(t):=W_2^2(\mu_t,\mu^*)+W_2^2(\nu_t,\nu^*)$.
For fixed $0<\beta<\alpha_0<1$, Lemma~\ref{lemma-smoothing} and
\eqref{eq-C^2bound} give uniform $C^{1,\alpha_0}$ bounds for the evolving
and equilibrium densities when $t\ge1$. Lemma~\ref{lemma-interpolation}
and $C^{1,\beta}\hookrightarrow W^{1,p}$ therefore yield $r>0$,
independent of the initial data, such that
\[
E(t)\le r^2
\quad\Longrightarrow\quad
\|\mu_t-\mu^*\|_{W^{1,p}}+\|\nu_t-\nu^*\|_{W^{1,p}}
\le\varepsilon_0/2,\qquad t\ge1.
\]
Choose $\lambda'\le0$ satisfying the assumptions of
Lemma~\ref{lemma-stability} and put $\delta=e^{\lambda'}r$.
If $W_2(\mu_0,\mu^*)+W_2(\nu_0,\nu^*)<\delta$, that lemma gives
$E(1)\le e^{-2\lambda'}E(0)<r^2$.
Define the first exit time $T_*:=\inf\{t\ge1:E(t)\ge r^2\}$.
By continuity, $T_*>1$, and Proposition~\ref{prop-main} gives
\[
E(t)\le e^{-2\lambda(t-1)}E(1)\le E(1)<r^2,
\qquad 1\le t<T_*.
\]
If $T_*<\infty$, continuity would imply $E(T_*)\le E(1)<r^2$,
contradicting its definition. Thus $T_*=\infty$.
Combining this decay with $E(t)\le e^{-2\lambda't}E(0)$ on $[0,1]$
from Lemma~\ref{lemma-stability}, we obtain
\[
E(t)\le e^{2(\lambda-\lambda')}e^{-2\lambda t}E(0),\qquad t\ge0.
\]
Since $\sqrt{E(t)}\le W_2(\mu_t,\mu^*)+W_2(\nu_t,\nu^*)
\le\sqrt{2E(t)}$, this proves \eqref{eq-maindecay} with
$C=\sqrt2e^{\lambda-\lambda'}$.
Finally, for any $0<\alpha<\gamma<1$, Lemma~\ref{lemma-smoothing}
and \eqref{eq-C^2bound} give uniform $C^{1,\gamma}$ bounds for $t\ge1$;
Lemma~\ref{lemma-interpolation} then upgrades the Wasserstein convergence
to convergence in $C^{1,\alpha}$.
\end{proof}

\section{Particle stability: proof of Theorem~\ref{thm-particle}}

The general idea for Theorem~\ref{thm-particle} is that, by classical propagation of chaos results, for large $N$, the empirical distributions $(\mu_t^N,\nu_t^N)$ will be close to their mean-field limit~\eqref{eq-main}. Up to a vanishing error as $N\rightarrow\infty$, the former thus essentially inherits the deterministic contraction of the latter, established in Theorem~\ref{thm-2}, \emph{as long as} the particles remain in the neighborhood of $(\mu^*,\nu^*)$ where this result applies. The main point is thus to control the exit time from this neighborhood. Thanks to the contraction, this follows from finite-time concentration inequalities and the Markov property.

Let us make this precise. To avoid confusion, in this section, write $\delta_0, C_0,\lambda_0$ the parameters $\delta,C,\lambda$ provided by Theorem~\ref{thm-2}. For $(\mu,\nu)\in\mathcal P(\mathcal X)\times\mathcal P(\mathcal Y)$, denote
\[r(\mu,\nu)= W_2(\mu,\mu^*) + W_2(\nu,\nu^*),\qquad \mathcal D = \{(\mu,\nu)\in\mathcal P(\mathcal X)\times\mathcal P(\mathcal Y),\ r(\mu,\nu) < \delta_0\}\,.\]
 We fix $T>0$ large enough so that $C_0 e^{-\lambda_0 T} \le 1/2$ and set 
\[
\tau := \inf\{m\ge 0:\ (\mu_{mT}^N,\nu_{mT}^N)\notin\mathcal D\}\in\mathbb N\cup\{\infty\}.
\]
\paragraph{Synchronous coupling.} Since it is one of the main tools in the proofs of both Proposition~\ref{prop:stopped_expectations} and Lemma~\ref{lem:HP}, let us first detail the construction of a synchronous coupling between interacting particles~\eqref{eq:particle} and independent McKean-Vlasov processes initially distributed according to the empirical distribution of the particles. With respect to the most classical situation, we do not assume that the interacting particles are initially i.i.d. 

For a fixed initial condition $(x_1,\dots,x_N,y_1,\dots,y_N) \in (\mathbb T^n)^{2N}$ of the particle system~\eqref{eq:particle}, let $(\mu_t^{\rm mf},\nu_t^{\rm mf})$ be the mean-field solution of \eqref{eq-main} with initial condition
$(\mu_0^{\rm mf},\nu_0^{\rm mf})=(\mu_0^N,\nu_0^N)$. Let $(\tilde X_0^i,\tilde Y_0^i)_{1\le i\le N}$ be independent variables with $\tilde X_0^i \sim \mu_0^N$ and $\tilde Y_0^i \sim \nu_0^N$, and denote by
\begin{equation}
    \label{eq:loc-tilde-empiric}
\tilde\mu_0^N:=\frac1N\sum_{i=1}^N\delta_{\tilde X_0^i},\qquad 
\tilde\nu_0^N:=\frac1N\sum_{j=1}^N\delta_{\tilde Y_0^j}
\end{equation}
the resampled empirical measures. Let $\sigma$ and $\sigma'$ be permutations minimizing the empirical quadratic costs
\begin{equation}
    \label{eq:sigma}
    \frac1N\sum_{i=1}^N |X_0^{\sigma(i)}-\tilde X_0^i|^2
= W_2(\mu_0^N,\tilde\mu_0^N)^2,
\qquad
\frac1N\sum_{j=1}^N |Y_0^{\sigma'(j)}-\tilde Y_0^j|^2
= W_2(\nu_0^N,\tilde\nu_0^N)^2.
\end{equation}
 Consider the processes having initial condition $\tilde X_0^i,\tilde Y_0^i$ and solving
\begin{equation}\label{eq:particle-NL}
\begin{cases}
d\tilde X_t^i \;=\; - \int_{\mathcal Y} \nabla_x f (\tilde X_t^i,y)\, d\nu_t^{\rm mf}(y)\,dt + \sqrt{2}\,dB_t^{\sigma(i)},\\[1mm]
d\tilde Y_t^j \;=\; +\int_{\mathcal X} \nabla _yf(x,\tilde Y_t^j)\, d\mu_t^{\rm mf}(x)\,dt + \sqrt{2}\,dW_t^{\sigma'(j)},
\end{cases}
\qquad i,j=1,\dots,N,
\end{equation}
where $B^i,W^j$ are the same Brownian motions as in~\eqref{eq:particle}. Since these Brownian motions are independent of the initial conditions, $(B^{\sigma(i)},W^{\sigma'(j)})_{1\le i,j\le N}$ are still independent Brownian motions independent from the initial conditions. As a consequence, the processes $(\tilde X^i,\tilde Y^j)_{1\le i,j\le N}$ are independent processes, and by uniqueness of the solution of the Kolmogorov equation associated to~\eqref{eq:particle-NL}, $\tilde X_t^i \sim \mu_t^{\rm mf}$ and $\tilde Y_t^i \sim \nu_t^{\rm mf}$ for all $t\geq 0, 1\le i \le N$.
Denote their empirical measures by
\[
\tilde\mu_t^N:=\frac1N\sum_{i=1}^N\delta_{\tilde X_t^i},
\qquad
\tilde\nu_t^N:=\frac1N\sum_{j=1}^N\delta_{\tilde Y_t^j}.
\]

In the proofs of both Proposition~\ref{prop:stopped_expectations} and Lemma~\ref{lem:HP}, the interest of this construction is to bound
\begin{equation}\label{eq:triangular-W2}
W_2(\mu_t^N, \mu_t^{\rm mf}) \le W_2(\mu_t^N,\tilde \mu_t^N) + W_2(\tilde \mu_t^N, \mu_t^{\rm mf}) \,,    
\end{equation}
and similarly for $W_2(\nu_t^N, \nu_t^{\rm mf}) $. Since the variables $(\tilde X_t^i)_{1\le i\le N}$ are i.i.d. with law $\mu_t^{\rm mf}$, the expectation and probability of deviations of the second term are controlled by known results, e.g. from \cite{fournier2015rate} or \cite{bolley2007quantitative}. For the first term, using that $(X_t^{\sigma(I)},\tilde X_t^I)$ with $I$ uniformly distributed over $1,\dots,N$ independent from the other variables gives a coupling of $\mu_t^N$ and $\tilde \mu_t^N$,  
\begin{equation}
    \label{eq:boundW2-mu-tildemu}
    W_2(\mu_t^N,\tilde \mu_t^N)^2 \le \frac1N \sum_{i=1}^N |X_t^{\sigma(i)}-\tilde X_t^i|^2\,.
\end{equation}
The right-hand side is controlled as follows. Since \(X_t^{\sigma(i)}\) and
\(\tilde X_t^i\) are driven by the same Brownian motion, 
\[
\begin{aligned}
&d |X_t^{\sigma(i)}-\tilde X_t^i|^2=
2(X_t^{\sigma(i)}-\tilde X_t^i)\cdot
\left(
\int_{\mathcal Y}\nabla_x f(\tilde X_t^i,y)\,d\nu_t^{\rm mf}(y)
-\frac1N\sum_{j=1}^N\nabla_x f(X_t^{\sigma(i)},Y_t^{\sigma'(j)})
\right)dt .
\end{aligned}
\]
Adding and subtracting
\(\frac1N\sum_{j=1}^N\nabla_x f(\tilde X_t^i,\tilde Y_t^j)\), and using that
\(\nabla_x f\) is \(\|\nabla^2 f\|_\infty\)-Lipschitz, 
\[
\begin{aligned}
d |X_t^{\sigma(i)}-\tilde X_t^i|^2
&\le
2\|\nabla^2 f\|_\infty |X_t^{\sigma(i)}-\tilde X_t^i| \\
&\quad\times
\left(
|X_t^{\sigma(i)}-\tilde X_t^i|
+\frac1N\sum_{j=1}^N|Y_t^{\sigma'(j)}-\tilde Y_t^j|
+W_2(\tilde\nu_t^N,\nu_t^{\rm mf})
\right)dt .
\end{aligned}
\]
Reasoning similarly with $Y_t^{\sigma'(i)}-\tilde Y_t^i$, summing these bounds over $i$ and writing 
\[{\rm e}(t) = \frac1N\sum_{i=1}^N  \Bigr(
|X_t^{\sigma(i)}-\tilde X_t^i|^2 +|Y_t^{\sigma'(i)}-\tilde Y_t^i|^2 \Bigr)
\]
leads with elementary computations to
\begin{equation}\label{eq:de(t)}
    d {\rm e}(t) \le \|\nabla^2 f\|_\infty \left(5{\rm e}(t)  + W_2(\tilde \mu_t^N, \mu_t^{\rm mf}) ^2 +  W_2(\tilde \nu_t^N, \nu_t^{\rm mf}) ^2\right) dt\,.
\end{equation}
Moreover, thanks to~\eqref{eq:sigma},
\begin{equation}
    \label{eq:e(0)}
{\rm e}(0) = W_2(\mu_0^N,\tilde\mu_0^N)^2
+ W_2(\nu_0^N,\tilde\nu_0^N)^2 = W_2(\mu_0^{\rm mf},\tilde\mu_0^N)^2
+ W_2(\nu_0^{\rm mf},\tilde\nu_0^N)^2  \,.
\end{equation}
Again, bounds on the expectation or probability of deviation ${\rm e}(0)$ and of the last two terms of~\eqref{eq:de(t)} are known, which by Grönwall's Lemma yields similar bounds on ${\rm e}(t)$, ultimately controlling $W_2(\mu_t^N,\tilde \mu_t^N)^2  $ thanks to~\eqref{eq:boundW2-mu-tildemu}, and similarly for $W_2( \nu_t^N,\tilde \nu_t^N)^2$, as will be detailed in the proofs below.

The first step is the following:   
 
\begin{proposition}[Stopped expectations]\label{prop:stopped_expectations}
There exists $\eta,C'>0$  such that {for all $N\ge1$ and $m\ge0$} and $s\in[0,T]$,
\begin{equation}
    \label{eq:stopped-expect}
    \mathbb E \left ( r(\mu_{mT+s}^N,\nu_{mT+s}^N) \mathbf 1_{\{\tau>m\}}\right)  \le C_0 e^{- \lambda_0 s}  \mathbb E \left ( r(\mu_{mT}^N,\nu_{mT}^N) \mathbf 1_{\{\tau>m\}}\right)  + C' N^{-\eta}\,.
\end{equation}
\end{proposition}
This is proven by inserting between $(\mu_{mT+s}^N,\nu_{mT+s}^N)$ and $(\mu^*,\nu^*)$, via the triangular inequality, the mean-field solution of~\eqref{eq-main} with initial condition $(\mu_{mT}^N,\nu_{mT}^N)$ at time $mT$, using a synchronous coupling and the results of \cite{fournier2015rate} to control the first term and Theorem~\ref{thm-2} for the second. The complete proof follows.

\begin{proof}
    [Proof of Proposition~\ref{prop:stopped_expectations}]
By the Markov property and working conditionally to $(\mu_{mT}^N,\nu_{mT}^N)$, it is sufficient to prove that there exists $\eta,C'>0$ such that for any fixed initial condition $(x_1,\dots,x_N,y_1,\dots,y_N)$  with empirical distributions $(\mu_0^N,\nu_0^N) $ satisfying $r(\mu_0^N,\nu_0^N) < \delta_0$ (corresponding to the condition that $\tau >0$), all $N \ge 1$ and $s\in[0,T]$,
\begin{equation}
    \label{eq:stopped-expect2}
    \mathbb E \left ( r(\mu_{s}^N,\nu_{s}^N) \right)  \le C_0 e^{- \lambda_0 s}    r(\mu_{0}^N,\nu_{0}^N)    + C' N^{-\eta}\,.
\end{equation}
In the rest of the proof, we thus consider such a fixed initial data with $r(\mu_0^N,\nu_0^N) < \delta_0 $. For $s\in[0,T]$, by  the triangular inequality, we bound
\begin{align*}
    r(\mu_s^N,\nu_s^N) & \le r( \mu_s^{\rm mf}, \nu_s^{\rm mf}) + W_2(\mu_s^N,\mu_s^{\rm mf}) + W_2(\nu_s^N, \nu_s^{\rm mf}) \\
    & \le C_0 e^{-\lambda_0 s} r(\mu_0^N,\nu_0^N) + \sqrt{2\rm{e}(s)} + W_2(\tilde \mu_s^N, \mu_s^{\rm mf}) + W_2(\tilde \nu_s^N, \nu_s^{\rm mf}) ,
\end{align*}
where we used Theorem~\ref{thm-2} for the first term (with the fact that $(\mu_0^{\rm mf},\nu_0^{\rm mf})= (\mu_0^N,\nu_0^N)$ by design) and~\eqref{eq:triangular-W2} and~\eqref{eq:boundW2-mu-tildemu} for the second term. Thanks to~\cite[Theorem 1]{fournier2015rate} (using that a moment of any order is bounded uniformly over all probability measures over the compact torus), there exists $C,\eta>0$ (which depends only on the dimension $n$) such that 
\[\mathbb E \left[W_2(\tilde \mu_s^N, \mu_s^{\rm mf})^2 + W_2(\tilde \nu_s^N, \nu_s^{\rm mf})^2\right] \leq C N^{-2\eta},\]
for all $s\in[0,T]$. Taking the expectation in~\eqref{eq:de(t)}, applying the Grönwall Lemma and finally Jensen's inequality yields 
\begin{align*}
    \mathbb E \left[ \sqrt{2\rm{e}(s)} + W_2(\tilde \mu_s^N, \mu_s^{\rm mf}) + W_2(\tilde \nu_s^N, \nu_s^{\rm mf}) \right] &\le C(T) N^{-\eta}
\end{align*}
uniformly over $s\in[0,T]$ for some constant $C(T)$, which concludes the proof.

\end{proof}

\begin{proof}
    [Proof of Theorem~\ref{thm-particle}]
We first estimate the process stopped before it exits the stability neighborhood. On \(\{\tau>m\}\), the empirical measures stay in \(\mathcal D\) at the discrete times \(0,T,\ldots,mT\), and hence the one-step contraction can be iterated. The stopping is then removed by controlling the exit probability \(\mathbb P(\tau\le m)\).
The immediate consequence of Proposition~\ref{prop:stopped_expectations} is that, for all $N\ge 1$ and $t\ge 0$, taking $m= \lfloor t/T\rfloor$\footnote{\(\lfloor t/T\rfloor\) denotes the largest integer not exceeding \(t/T\).} and using that $\mathbf 1_{\{\tau>k+1\}} \leqslant \mathbf 1_{\{\tau > k\}}$ for all $k\in\mathbb N$,  
\begin{align*}
     \mathbb E \left ( r(\mu_{t}^N,\nu_{t}^N)  \mathbf 1_{\{ \tau > {m} \}} \right)  &\le C_0 \mathbb E \left ( r(\mu_{mT}^N,\nu_{mT}^N)  \mathbf 1_{\{\tau > {m}\}} \right) +  C' N^{-\eta}  \\
     &\le C_0 2^{-m} \mathbb E \left ( r(\mu_{0}^N,\nu_{0}^N)  \mathbf 1_{\{\tau > 0\}} \right) + (2C_0+1) C' N^{-\eta} \\
     &\le 2C_0 e^{-\lambda_1 t} \mathbb E \left ( r(\mu_{0}^N,\nu_{0}^N)  \mathbf 1_{\{\tau > 0\}} \right) + (2C_0+1) C' N^{-\eta}
\end{align*}
with $\lambda_1 = \ln(2)/T$. Writing $D_{\mathbb T^n}$  for the diameter of $\mathbb T^n$, we immediately deduce that
\begin{align*}
  \mathbb E \left ( r(\mu_{t}^N,\nu_{t}^N) \right)   &\le \mathbb E \left ( r(\mu_{t}^N,\nu_{t}^N)  \mathbf 1_{\{\tau > {m} \}} \right)  + 2D_{\mathbb T^n} \mathbb P(\tau \le m ) \\
  &\le C_1 e^{-\lambda_1 t} \mathbb E \left ( r(\mu_{0}^N,\nu_{0}^N) \right) + (2C_0+1) C' N^{-\eta} +  2D_{\mathbb T^n} \mathbb P(\tau \le m ) \,.
\end{align*}
It only remains to control the last term to conclude the proof of Theorem~\ref{thm-particle}. First, 
\[\mathbb P(\tau \le m ) = \mathbb P(\tau \le m , \tau >0 ) + \mathbb P(\tau = 0 )\,.\]
The last term is bounded by $R_0 N^{-\eta}$ according to~\eqref{eq:initial_particle} (choosing $\delta = \delta_0$ in Theorem~\ref{thm-particle}). The proof is then concluded with Proposition~\ref{prop:exit-time} stated below.
\end{proof}

\begin{proposition}
    [Metastable confinement]\label{prop:exit-time}
    There exist $b>0$ and $K<\infty$ such that for any $N\ge 1$ and $a >0$,
    \[\mathbb P(\tau \le \lfloor e^{aN}\rfloor  , \tau >0 ) \le Ke^{-(b-a)N}.\]
\end{proposition}

\begin{proof}
    [Proof of Proposition~\ref{prop:exit-time}] Set $\varepsilon = \delta_0/2$. For $k \in \mathbb N$, consider the events
    \[E_k = \Big\{    W_2(\mu_{(k+1)T}^N,\mu_{k,{(k+1)T}}^{\rm mf})+ W_2(\nu_{(k+1)T}^N,\nu_{k,{(k+1)T}}^{\rm mf})\le \varepsilon\Big \}, \qquad F_k = \left\{ (\mu_{kT}^N,\nu_{kT}^N) \in \mathcal D \right\}, \]
where $(\mu_{k,t}^{\rm mf},\nu_{k,t}^{\rm mf})$ is defined as previously $(\mu_{t}^{\rm mf},{\nu_{t}^{\rm mf}})$ but starting at time $kT$ (i.e. they are the solutions of~\eqref{eq-main} initialized at $(\mu_{kT}^N,\nu_{kT}^N)$ at time $kT$). Under the event $F_{k-1} \cap E_{k-1}$, applying Theorem~\ref{thm-2} and the triangular inequality,
\begin{align*}
r(\mu_{kT}^N,\nu_{kT}^N) & \leqslant r(\mu_{k-1,kT}^{\rm mf},\nu_{k-1,kT}^{\rm mf}) + W_2(\mu_{kT}^N,\mu_{k-1,kT}^{\rm mf})+ W_2(\nu_{kT}^N,\nu_{k-1,kT}^{\rm mf})     \\
&\le \frac12 r(\mu_{(k-1)T}^N,\nu_{(k-1)T}^N) + \varepsilon \ < \ \delta_0,
\end{align*}
namely $F_k$ holds. As a consequence, using that $\{\tau>0\} = F_0$, for any $m\ge1$ 
\[\mathbb P \left( \tau >0, \tau \le m \right)  \le \mathbb P\left(\tau >0 , \cup_{k=0}^{m-1} E_k^c\right) \le \sum_{k=0}^{m-1} \mathbb P(E_k^c) \le Km e^{-bN}\]
thanks to Lemma~\ref{lem:HP} and the Markov property at time $kT$, which concludes the proof.
\end{proof}

The finite-horizon estimate used in the preceding proof is stated and proved next. For a fixed initial condition $(x_1,\dots,x_N,y_1,\dots,y_N) \in (\mathbb T^n)^{2N}$ of~\eqref{eq:particle}, let $(\mu_t^{\rm mf},\nu_t^{\rm mf})$ be the mean-field solution of \eqref{eq-main} with initial condition
$(\mu_0^{\rm mf},\nu_0^{\rm mf})=(\mu_0^N,\nu_0^N)$.

\begin{samepage}

The next lemma combines synchronous coupling with concentration for independent processes.


\begin{lemma}
    [Finite-horizon exponential concentration]\label{lem:HP}
Fix $T>0$ and $\varepsilon>0$.
There exist $b=b(T,\varepsilon)>0$ and $K=K(T,\varepsilon)<\infty $ such that for any initial data with $(\mu_0^N,\nu_0^N)\in \mathcal P(\mathcal X)\times\mathcal P(\mathcal Y)$,
the particle system is $\varepsilon$--close to the mean-field solution 
at time $T$ with probability at least $1-Ke^{-bN}$:
\begin{equation}
    \label{eq:concentration}
\mathbb P\Big(
W_2(\mu_{T}^N,\mu_{T}^{\rm mf})+ W_2(\nu_{T}^N,\nu_{T}^{\rm mf})>\varepsilon\Big)\ \le\ Ke^{-bN}.
\end{equation}
\end{lemma}
\end{samepage}

\begin{proof}
    [Proof of Lemma~\ref{lem:HP}]
Let $D$ be the diameter of $\mathbb T^n$ and $L=\|\nabla^2 f\|_\infty$. Set
$w_t:=W_1(\tilde\mu_t^N,\mu_t^{\rm mf})+W_1(\tilde\nu_t^N,\nu_t^{\rm mf})$ and
$G:=w_0+w_T+\int_0^T w_s\,ds$.
Since $W_2^2\le D\,W_1$ on $\mathbb T^n$ (see the proof of Lemma~\ref{lemma-interpolation}), combining \eqref{eq:triangular-W2}, \eqref{eq:boundW2-mu-tildemu}, \eqref{eq:e(0)} and Grönwall's lemma applied pathwise to \eqref{eq:de(t)} yields
\[
\big(W_2(\mu_T^N,\mu_T^{\rm mf})+W_2(\nu_T^N,\nu_T^{\rm mf})\big)^2
\le 3\big(2{\rm e}(T)+D\,w_T\big)\le C_1G
\]
for some $C_1=C_1(T,L,D)$. The $2N$ paths $\tilde X^i,\tilde Y^j$ are independent, and replacing one of them moves one atom of mass $1/N$, which changes $w_t$ by at most $D/N$. Hence $G$ has bounded differences $(2+T)D/N$, and McDiarmid's inequality \citep{mcdiarmid1989method} gives
\[
\mathbb P\big(G\ge \mathbb E G+s\big)\le \exp\Big(-\frac{Ns^2}{(2+T)^2D^2}\Big),\qquad s>0.
\]
Moreover, $W_1\le W_2$, so the bound used in the proof of Proposition~\ref{prop:stopped_expectations} gives $\mathbb E G\le C_2N^{-\eta}$, uniformly in the initial configuration. Let $s=\varepsilon^2/(2C_1)$ and choose $N_0$ with $C_2N_0^{-\eta}\le s$. For $N\ge N_0$, the event in \eqref{eq:concentration} is contained in $\{G\ge \mathbb E G+s\}$; this proves the lemma with $b=s^2/((2+T)^2D^2)$ and $K=e^{bN_0}$.
\end{proof}

\section{Conclusion}\label{sec-conclusion}

We proved local exponential stability of the mean-field Langevin
descent--ascent dynamics near the mixed Nash equilibrium. A spectral gap
for the Wasserstein Hessian gives local displacement convex--concavity,
which yields quantitative exponential convergence in $W_2$ through an
EVI-type contraction argument. The associated finite-particle system
inherits this stability up to times exponential in the number of
particles. Together with the counterexamples to global convergence,
these results describe a local basin of attraction that persists under
finite-particle approximation.

\paragraph{Other state spaces.}
The same mechanism extends to closed, connected Riemannian manifolds,
possibly of different dimensions for the two players. The entropy
Hessian acquires a Ricci-curvature term, but the weighted Bochner
identity combines it with the potential Hessian to give
$\|\Delta\phi-\langle\nabla V,\nabla\phi\rangle\|_{L^2(\pi)}^2$
for $d\pi\propto e^{-V}\,d\mathrm{vol}$. Thus the spectral-gap argument
in Lemma~\ref{lemma-convconc} still applies. Transport perturbations
use the exponential map, while elliptic and parabolic estimates are
handled in local charts; see
\citet{villani2008optimal,AGS08,bakry2013analysis,Lott2008GeometricCalculationsWasserstein}.

On $\R^n$, a natural starting point is confinement that ensures a
Poincar\'e inequality for the equilibrium measures, the condition
underlying the spectral gap in Remark~\ref{remark-alternative}.
Further steps require weighted regularity and transport-map stability
estimates, since the uniform positive lower bounds for densities used
on the torus are unavailable. We leave these extensions to future work,
along with sharper asymptotics, optimal convergence constants, and
annealing schemes that approach the unregularized game while
preserving local stability.


\acks{This research was supported by the National Research Foundation of Korea (NRF) grants Nos. 2022R1C1C1013511, RS-2023-00219980, and RS-2026-25613008.}






\bibliography{ref}

\tableofcontents

\newpage
\appendix

\section{Preliminaries for the stability analysis}
\label{app-prelim}
We collect the analytic and transport tools used in the proof of
Theorem~\ref{thm-2}. Theorem~\ref{thm:main_appendix} supplies well-posedness
for its arbitrary measure-valued initial data. Supplementary proofs and
discussion are collected in Appendix~\ref{appendix-detailedproofs}.


\subsection{Notation and analytic estimates}
\label{sec:notations}\label{sec:sobolev}
Functions and vector fields on $\Tn=\R^n/\Z^n$ are identified with their
periodic lifts. We write $\mathcal P_2(\Tn)=\mathcal P(\Tn)$ and identify
absolutely continuous measures with their densities when convenient.
All unqualified function spaces are on $\Tn$, with their standard norms.

The regularity arguments use the following Sobolev--Morrey embeddings
\citep[Ch.~5.6]{evans2010partial}:
\begin{equation}\label{eq:analytic-embeddings}
\begin{aligned}
W^{1,p}&\hookrightarrow L^{np/(n-p)} &&(1\le p<n),\\
W^{1,n}&\hookrightarrow L^q &&(n>1,\ 1\le q<\infty),\\
W^{k+1,p}&\hookrightarrow C^{k,1-n/p} &&(n<p<\infty,\ k=0,1).
\end{aligned}
\end{equation}
For $n=1$, $W^{1,1}\hookrightarrow L^\infty$. These embeddings,
applied also to first derivatives, allow the spatial
integrability exponent to be increased in the proof of
Lemma~\ref{lemma-smoothing}; Morrey's embedding also gives the
$C^0$ control used in Lemma~\ref{lemma-locconvconc}. The inclusion
$C^{1,\alpha}\hookrightarrow W^{1,p}$, $1\le p\le\infty$, is used in
the final bootstrap for Theorem~\ref{thm-2}. The Banach algebra property
of $W^{k,p}$, $k\ge1$ and $n<p<\infty$, is used in
Lemmas~\ref{lemma:sobolev_stability} and~\ref{lemma-IFT2}.

The compactness arguments require only
\begin{equation}\label{eq:compact-embeddings}
W^{2,2}\Subset W^{1,2},\qquad
C^{1,\alpha}\Subset C^{1,\beta}\quad(0<\beta<\alpha<1),
\end{equation}
where $\Subset$ denotes compact inclusion. In particular, a bounded
sequence in $W^{2,2}$ has a subsequence converging weakly in $W^{2,2}$
and strongly in $W^{1,2}$, with the usual weak lower semicontinuity of
the norm. These are the Rellich and Arzel\`a--Ascoli compactness
properties used in Lemmas~\ref{lemma-convconc} and
\ref{lemma-interpolation}; see \citet[Chs.~3 and 9]{Brezis2011} and
\citet[Lemma~6.33]{GT}.

\subsection{Wasserstein formulas and convexity conventions}
\label{sec:ot}\label{app:disp-convexity}
For $\mu\ll dx$ and $\nu\in\mathcal P(\Tn)$, the optimal map has the
form $T_{\mu\to\nu}=\mathrm{id}+\nabla\phi$ modulo $\Z^n$, with
$\phi$ periodic. In expressions such as $T-\mathrm{id}$ we use this
optimal displacement. The associated constant-speed geodesic is 

\[
\mu_s=(\mathrm{id}+s\nabla\phi)_\#\mu ,\,\, s\in[0,1], \quad
W_2^2(\mu,\nu)=\int|\nabla\phi|^2\,d\mu.
\]
We use the standard tangent space
$\Tan_\mu\mathcal P_2=\overline{\{\nabla\psi:\psi\in C^\infty(\Tn)\}}^{L^2(\mu)}$;
the tangent velocity $v_t$ of an absolutely continuous curve satisfies
$\partial_t\mu_t+\nabla\cdot(\mu_t v_t)=0$.
On the compact torus, narrow convergence is equivalent to $W_2$
convergence. We refer to \citet[Chs.~6--8]{AGS08} and
\citet{villani2008optimal} for these facts.

\begin{theorem}[First variation, {\cite[Theorem~8.4.7]{AGS08}}]
\label{thm:first_variation_W2}\leavevmode\newline
Let $(\mu_t)$ be an absolutely continuous curve with tangent velocity
$v_t$, and assume $\mu_t\ll dx$ for a.e.\ $t$. For fixed
$\nu\in\mathcal P(\Tn)$, the squared distance is absolutely continuous
and, for a.e.\ $t$,
\[
\frac{d}{dt}\frac12W_2^2(\mu_t,\nu)
=\bigl\langle v_t,\mathrm{id}-T_{\mu_t\to\nu}\bigr\rangle_{L^2(\mu_t)}.
\]
\end{theorem}

We use $\lambda$-displacement convexity in the usual geodesic sense
\citep[Definition~9.1.1]{AGS08}, and call $\mathcal G$
$\lambda$-displacement concave when $-\mathcal G$ is
$\lambda$-displacement convex. Along smooth constant-speed geodesics,
the convexity condition is
$(\mathcal G(\mu_s))''\ge\lambda W_2^2(\mu_0,\mu_1)$.
Its above-tangent form is
\[
\mathcal G(\nu)-\mathcal G(\mu)
\ge\bigl\langle\operatorname{grad}_{W_2}\mathcal G(\mu),
T_{\mu\to\nu}-\mathrm{id}\bigr\rangle_{L^2(\mu)}
+\frac{\lambda}{2}W_2^2(\mu,\nu),
\]
whenever the quantities are well defined; see \citet[(10.1.7)]{AGS08}.
Together with Theorem~\ref{thm:first_variation_W2}, this inequality
gives the EVI estimates in Proposition~\ref{prop-main} and
Lemma~\ref{lemma-stability-cc}, which in turn proves
Lemma~\ref{lemma-stability}.

\subsection{Well-posedness of MFL-DA }\label{sec:welposedness}

In this subsection, we introduce the well-posedness of the coupled mean-field PDE \eqref{eq-main}.
A probabilistic approach goes back at least to \cite{funaki1984certain}, who treated nonlinear
Fokker-Planck/McKean-Vlasov type equations via the martingale problem.
In \cite{domingoenrich2020}, existence and uniqueness are established for a coupled McKean-Vlasov SDE,
and its time-marginals yield a unique measure-valued solution
$(\mu_t,\nu_t)\in C([0,T];\mathcal P_2(\mathcal{X}))\times C([0,T];\mathcal P_2(\mathcal{Y}))$ of \eqref{eq-main}.
From a variational viewpoint, \cite{isobe2025} constructs gradient descent--ascent flows through a
minimizing-maximizing movement scheme and proves well-posedness in the EVI sense; whenever
\eqref{eq-main} admits a Wasserstein GDA/EVI formulation, this provides another route to
measure-valued well-posedness.
Under stronger smoothness assumptions on the interaction kernel $f$, one may also work at the
level of densities and obtain $C([0,T];L^p)$ (mild) solutions via a Duhamel formulation
(cf. \cite{CJS25}).

\medskip
The proof strategy is adapted from the path-space fixed point argument of \cite{domingoenrich2020},
specialized to the present torus setting and to our sign/notational conventions.
For the reader's convenience, we first collect several preliminary definitions and results that will be used throughout the proof.

\textit{Setting and distributional formulation.}
Let $M:=\mathbb T^n$ be the flat torus with geodesic distance $d$ and Lebesgue measure $dx$. 
Since $f\in C^2(M\times M)$ and $M\times M$ is compact, $\nabla_1 f$ and $\nabla_2 f$ are globally Lipschitz:
if we set $L:=\|\nabla^2 f\|_{L^\infty(M\times M)}$, then for all $(x,y),(x',y')\in M\times M$,
\begin{equation}\label{eq:ass_Lip2}
\|\nabla_1 f(x,y)-\nabla_1 f(x',y')\|
+\|\nabla_2 f(x,y)-\nabla_2 f(x',y')\|
\le L\big(d(x,x')+d(y,y')\big).
\end{equation}

A narrowly continuous curve $(\mu_t,\nu_t)_{t\in[0,T]}\subset\mathcal P(M)\times\mathcal P(M)$
is a distributional solution of \eqref{eq-main} if for every $\varphi\in C^2(M)$ and all $t\in[0,T]$,
\begin{align}
\int_M \varphi\,d\mu_t - \int_M \varphi\,d\mu_0
&=\int_0^t \int_M \Delta \varphi\,d\mu_s\,ds\notag\\
&\quad-\int_0^t \int_M \nabla\varphi(x)\cdot \nabla_1\Big(\int_M f(x,y)\,d\nu_s(y)\Big)\,d\mu_s(x)\,ds,
\label{eq:weak_mu2}\\
\int_M \varphi\,d\nu_t - \int_M \varphi\,d\nu_0
&=\int_0^t \int_M \Delta \varphi\,d\nu_s\,ds\notag\\
&\quad+\int_0^t \int_M \nabla\varphi(y)\cdot \nabla_2\Big(\int_M f(x,y)\,d\mu_s(x)\Big)\,d\nu_s(y)\,ds.
\label{eq:weak_nu2}
\end{align}

\begin{theorem}[Well-posedness and instantaneous smoothing]\label{thm:main_appendix}
Let $M:=\mathbb T^n$ and let $\mu_0,\nu_0\in\mathcal P(M)$ be arbitrary probability measures.
Assume $f\in C^2(M\times M)$.
Then for every $T>0$ the coupled system \eqref{eq-main}
admits a unique narrowly continuous distributional solution $(\mu_t,\nu_t)_{t\in[0,T]}$
in the sense of \eqref{eq:weak_mu2}--\eqref{eq:weak_nu2}.
Moreover, for every $\tau\in(0,T)$ there exist densities $u,v$ on $(\tau,T]\times M$ such that
$\mu_t=u(t,\cdot)\,dx$ and $\nu_t=v(t,\cdot)\,dx$ for all $t\in(\tau,T]$, and for every $p\in(1,\infty)$,
\begin{equation}\label{eq:max_reg_statement}
u,v\in L^p((\tau,T);W^{2,p}(M))\cap W^{1,p}((\tau,T);L^p(M)).
\end{equation}
\end{theorem}

\begin{proof}
We construct a fixed point on path space, identify its marginals with the
unique distributional solution, and then establish positive-time regularity.

\textit{Step 1: Fixed-point setting on path space.}
Let $\Omega=C([0,T],M)$ and write $\rho_t=(\mathrm{ev}_t)_\#\rho$ for the
time-$t$ marginal of $\rho\in\mathcal P(\Omega)$, where
$\mathrm{ev}_t(\omega)=\omega(t)$. Define
\begin{equation}\label{eq:W2t2}
W_{2,t}^2(\rho,\tilde\rho)
:=\inf_{\pi\in\Pi(\rho,\tilde\rho)}
\int_{\Omega\times\Omega}\sup_{0\le s\le t}d^2(\omega_s,\tilde\omega_s)
\,d\pi(\omega,\tilde\omega).
\end{equation}
For input laws $(\rho,\eta)$, freeze the drifts as
\[
b^\eta(t,x)=-\int_M\nabla_1 f(x,y)\,d\eta_t(y),\qquad
c^\rho(t,y)=\int_M\nabla_2 f(x,y)\,d\rho_t(x),
\]
and consider
\begin{align}
dX_t&=b^\eta(t,X_t)\,dt+\sqrt2\,dW_t,\qquad X_0\sim\mu_0,
\label{eq:SDE_frozen2}\\
dY_t&=c^\rho(t,Y_t)\,dt+\sqrt2\,d\bar W_t,\qquad Y_0\sim\nu_0.
\label{eq:SDE_frozenY2}
\end{align}
The drifts are bounded and uniformly Lipschitz in space, so these SDEs
have unique strong solutions, obtained by lifting the periodic drifts to
$\R^n$ and projecting modulo $\Z^n$. Denote their path laws by
$\Phi_1(\eta)$ and $\Phi_2(\rho)$, and set
$\Phi(\rho,\eta)=(\Phi_1(\eta),\Phi_2(\rho))$.

\medskip
\textit{Step 2: Picard fixed point.}
The key estimate controls the dependence of the frozen solutions on their
input laws.
\begin{lemma}\label{lem:stability_volterra}
There is $C=C(L,T)<\infty$ such that, for all
$(\rho,\eta),(\tilde\rho,\tilde\eta)\in\mathcal P(\Omega)^2$ and $t\in[0,T]$,
\begin{equation}\label{eq:stability2}
\begin{aligned}
&W_{2,t}^2(\Phi_1(\eta),\Phi_1(\tilde\eta))
 +W_{2,t}^2(\Phi_2(\rho),\Phi_2(\tilde\rho))\le C\int_0^t
\bigl(W_{2,s}^2(\eta,\tilde\eta)+W_{2,s}^2(\rho,\tilde\rho)\bigr)\,ds.
\end{aligned}
\end{equation}
\end{lemma}
\begin{proof}
We prove the estimate for $\Phi_1$; the other component is identical.
Integrating the Lipschitz bound \eqref{eq:ass_Lip2} for $\nabla_1 f$ 
against a coupling gives
\[
\|b^\eta(t,\cdot)-b^{\tilde\eta}(t,\cdot)\|_\infty
\le L W_2(\eta_t,\tilde\eta_t)\le L W_{2,t}(\eta,\tilde\eta).
\]
The last inequality holds because evaluation at time $t$ does not increase
the uniform distance between paths.
Couple the lifted SDEs using the same initial variable and Brownian motion,
and let $Z_t$ be the difference of their lifts. The noise cancels, and spatial
Lipschitz continuity and Young's inequality yield
\[
\frac{d}{dt}|Z_t|^2
\le 2L|Z_t|\bigl(|Z_t|+W_{2,t}(\eta,\tilde\eta)\bigr)
\le 3L|Z_t|^2+L W_{2,t}^2(\eta,\tilde\eta),\qquad Z_0=0.
\]
Gr\"onwall's inequality, followed by projection to $M$, gives
\[
W_{2,t}^2(\Phi_1(\eta),\Phi_1(\tilde\eta))
\le \mathbb E\sup_{s\le t}|Z_s|^2
\le Le^{3LT}\int_0^t W_{2,s}^2(\eta,\tilde\eta)\,ds.
\]
Adding the corresponding estimate for $\Phi_2$ proves the lemma.
\end{proof}

Starting from any pair of path laws, iterate
$(\rho^{(k+1)},\eta^{(k+1)})=\Phi(\rho^{(k)},\eta^{(k)})$ and set
$d_k(t)=W_{2,t}^2(\rho^{(k+1)},\rho^{(k)})
+W_{2,t}^2(\eta^{(k+1)},\eta^{(k)})$.
By \eqref{eq:stability2} and iteration,
\[
d_k(t)\le C\int_0^t d_{k-1}(s)\,ds
\le \frac{(Ct)^k}{k!}\,d_0(T),\qquad k\ge1.
\]
Thus $\sum_k\sqrt{d_k(T)}<\infty$, so the iterates are Cauchy.
Since $M$ is compact, $\Omega$ has a bounded complete uniform metric;
hence $(\mathcal P(\Omega),W_{2,T})$ is complete. The limit $(\rho,\eta)$
is a fixed point by \eqref{eq:stability2}. For two fixed points, their summed
squared distance $D(t)$ satisfies $D(t)\le C\int_0^t D(s)\,ds$ by the
same estimate; hence $D\equiv0$ by Gr\"onwall.

\medskip
\textit{Step 3: Identification and uniqueness of the PDE solution.}
The fixed point has narrowly continuous marginals
$(\mu_t,\nu_t)=(\rho_t,\eta_t)$ with the prescribed initial data.
Applying It\^o's formula to a test function $\varphi\in C^2(M)$ in
\eqref{eq:SDE_frozen2}--\eqref{eq:SDE_frozenY2} and taking expectations
gives \eqref{eq:weak_mu2}--\eqref{eq:weak_nu2}.

Conversely, let $(\tilde\mu_t,\tilde\nu_t)$ be any distributional solution
with the same initial data. Freeze the drifts along these marginals.
The superposition principle for the resulting linear Fokker--Planck
equations (see, e.g., \cite{trevisan2016}) gives laws
$(\tilde\rho,\tilde\eta)$ of weak solutions of the frozen SDEs with these
marginals. The frozen drifts are thus $b^{\tilde\eta}$ and $c^{\tilde\rho}$;
since they are bounded and spatially Lipschitz, uniqueness in law yields
\[
\tilde\rho=\Phi_1(\tilde\eta),\qquad
\tilde\eta=\Phi_2(\tilde\rho).
\]
These laws therefore form a fixed point of $\Phi$. Step 2 implies
$(\tilde\rho,\tilde\eta)=(\rho,\eta)$, proving uniqueness of the PDE solution.

\medskip
\textit{Step 4: Instantaneous regularization.}
For the fixed point, write $b_t=b^\eta(t,\cdot)$ and
$c_t=c^\rho(t,\cdot)$. Since both marginals are probability measures,
\begin{equation}\label{eq:drift_bounds_revised}
\sup_{t\in[0,T]}\bigl(\|b_t\|_{W^{1,\infty}}+
\|c_t\|_{W^{1,\infty}}\bigr)\le C\|f\|_{C^2}.
\end{equation}
Fix $0<\tau<T$ and $1<p<\infty$. Transition-density bounds for uniformly
elliptic diffusions with bounded drift give densities $u,v$ for positive
times, bounded in $L^p(M)$ on $[\tau/2,T]$. They satisfy
\begin{equation}\label{eq:density_form_revised}
\partial_t u-\Delta u=-\nabla\cdot(ub_t),\qquad
\partial_t v-\Delta v=-\nabla\cdot(vc_t).
\end{equation}
For completeness, the passage to \eqref{eq:max_reg_statement} uses two
interior heat-equation estimates \citep{lieberman1996second,Krylov2008}.
First, $ub\in L^p((\tau/2,T);L^p)$ gives
$u\in L^p((3\tau/4,T);W^{1,p})$ by the estimate for a divergence-form
source. Using \eqref{eq:drift_bounds_revised}, we can then write
\[
-\nabla\cdot(ub)=-b\cdot\nabla u-(\nabla\cdot b)u
\in L^p((3\tau/4,T);L^p).
\]
Maximal $L^p$ regularity on $(\tau,T)$ now yields the two bounds in
\eqref{eq:max_reg_statement} for $u$.
The same argument applies to $v$. The time cutoffs avoid imposing any
regularity on the initial measures; the bounds may deteriorate as
$\tau\downarrow0$.
\end{proof}

\subsection{Sobolev stability of optimal transport}
\label{sec:regularities}
Lemma~\ref{lemma:sobolev_stability} is used in
Lemma~\ref{lemma-locconvconc} to pass from Sobolev control of the
transport displacement to control of the density.
\begin{lemma}[Sobolev stability under push-forward]\label{lemma:sobolev_stability}
Let $p>n$ and let $k\ge1$ be an integer. Let $\mu=\rho\,dx\in\mathcal P_{ac}(\Tn)$ with
\[
\rho\in W^{k+1,p}(\Tn),\qquad 0<c\le \rho \le C <\infty\ \text{ a.e. on }\Tn .
\]
Let $T=id+\nabla\psi$ be the optimal transport map pushing $\mu$ to $\nu:=T_\#\mu$.
Assume that $\|\nabla\psi\|_{W^{k+1,p}}$ is sufficiently small so that $T$ is a $W^{k+1,p}$-diffeomorphism.
Then $\nu$ admits a density $\eta\in W^{k,p}(\Tn)$ and there exists a constant
$C_{\mathrm{st}}>0$, depending only on $n,p,k$, $c,C$, and $\|\rho\|_{W^{k+1,p}}$, such that
\[
\|\eta-\rho\|_{W^{k,p}} \le C_{\mathrm{st}}\;\|\nabla\psi\|_{W^{k+1,p}}.
\]
\end{lemma}

\begin{proof}
Write $v=\nabla\psi$ and $S=T^{-1}=id+u$. For
$\|v\|_{W^{k+1,p}}$ small, differentiating $T\circ S=id$ and using
the Sobolev algebra property gives
\[
\|u\|_{W^{k+1,p}}+\|\det(\nabla S)-1\|_{W^{k,p}}
\le C\|v\|_{W^{k+1,p}}.
\]
The change-of-variables formula gives a density
\begin{equation}\label{eq-changeofvar}
\eta=(\rho\circ S)\det(\nabla S)
=\frac{\rho\circ T^{-1}}{\det(\nabla T)\circ T^{-1}}.
\end{equation}
The maps $id+su$, $0\le s\le1$, are uniformly close to the identity
in $C^1$. The fundamental theorem of calculus and Sobolev composition
estimates therefore give
\[
\|\rho\circ S-\rho\|_{W^{k,p}}
=\left\|\int_0^1\nabla\rho(id+su)\cdot u\,ds\right\|_{W^{k,p}}
\le C\|\rho\|_{W^{k+1,p}}\|u\|_{W^{k,p}}.
\]
In particular, $\eta\in W^{k,p}$. Decomposing
\[
\eta-\rho=(\rho\circ S-\rho)\det(\nabla S)
+\rho\bigl(\det(\nabla S)-1\bigr)
\]
and applying the Banach algebra estimate yields
$\|\eta-\rho\|_{W^{k,p}}\le C_{\mathrm{st}}\|v\|_{W^{k+1,p}}$,
with the stated dependence of the constant.
\end{proof}

The converse estimate, obtained by the inverse function theorem,
is used in Proposition~\ref{prop-main} to turn $W^{1,p}$-closeness of
densities into $W^{2,p}$-control of the transport displacement.

\begin{lemma}[Sobolev stability, cf. {\cite[Lemma~5.1]{CJS25}}]\label{lemma-IFT2}

For $p>n$ and an integer $k\ge1$, let $\mu=\rho dx\in \cP_{ac}(\Tn)$ satisfy
\begin{enumerate}\item $0<C^{-1}\le \rho\le C$ a.e.\ on $\Tn$ for some $C<\infty$, and
\item $\rho\in W^{k+1,p}(\Tn)$.
\end{enumerate}
Then there exist $\varepsilon>0$ and $C'<\infty$, depending only on $n,p,k,C$ and
$\|\rho\|_{W^{k+1,p}}$, such that if $\mu_1=\rho_1dx\in \cP_{ac}(\Tn)$ satisfies
$\|\rho_1-\rho\|_{W^{k,p}}\le\varepsilon$, then the optimal transport map $T$ from
$\mu$ to $\mu_1$ can be written as $T=id+\nabla\psi$ with $\psi\in W^{k+2,p}(\Tn)$
(normalized by $\int_{\Tn}\psi\,dx=0$), and the induced gradient vector field
$\Phi:=\nabla\psi$ satisfies
\[
\|\Phi\|_{W^{k+1,p}} \le C'\,\|\rho_1-\rho\|_{W^{k,p}}.
\]
\end{lemma}
\begin{proof}
Set $X=\{\phi\in W^{k+2,p}:\int\phi\,dx=0\}$ and
$Y=\{h\in W^{k,p}:\int h\,dx=0\}$. For $\phi\in X$ close to zero,
define
\[
S_\phi=id+\nabla\phi,\qquad
J(\phi)=(\rho\circ S_\phi)\det(\nabla S_\phi).
\]
Change of variables gives $J(\phi)-\rho\in Y$. Since
$\rho\in W^{k+1,p}$ and $p>n$, Sobolev composition and product
estimates show that $J:X\supset U\to\rho+Y$ is $C^1$.
By the Piola identity $\nabla\cdot\operatorname{cof}(\nabla S_\phi)=0$,
\[
DJ(\phi)h=\nabla\cdot(B_\phi\nabla h),\qquad
B_\phi=(\rho\circ S_\phi)\operatorname{cof}(\nabla S_\phi),
\]
so $DJ(0)h=\nabla\cdot(\rho\nabla h)$. In a sufficiently small
$X$-ball, $B_\phi$ is uniformly elliptic; moreover,
$\nabla\cdot B_\phi=\det(\nabla S_\phi)(\nabla\rho)\circ S_\phi$,
so the coefficients in nondivergence form are uniformly bounded in
$W^{k,p}$. The Poincar\'e inequality, Lax--Milgram theorem, and
elliptic regularity \citep{GT} therefore make
$DJ(\phi):X\to Y$ an isomorphism with a uniformly bounded inverse: $$\|h\|_{W^{k+2,p}}
\le
C\|\operatorname{div}(B_\phi\nabla h)\|_{W^{k,p}},
\qquad \int h=0,$$ 
These bounds depend only on the data in the statement.

Apply the inverse function theorem at zero and continue the local
inverse along $\rho+t(\rho_1-\rho)$, $0\le t\le1$.
The inverse bound gives
$\|\partial_t\phi_t\|_X\le C\|\rho_1-\rho\|_{W^{k,p}}$ and hence
$\|\phi_t\|_X\le Ct\|\rho_1-\rho\|_{W^{k,p}}$.
For $\varepsilon$ small, this prevents exit from the above ball;
the same bound gives an $X$-limit at any finite endpoint, where the
inverse function theorem extends the curve. Thus $J(\phi)=\rho_1$
has a solution with $\|\phi\|_X\le C\|\rho_1-\rho\|_{W^{k,p}}$,
and $S_\phi{}_{\#}(\rho_1\,dx)=\rho\,dx$.

Since $\phi$ is periodic and $I+\nabla^2\phi>0$, the torus
optimality criterion identifies $S_\phi$ as the optimal map from
$\mu_1$ to $\mu$. By convex duality, its inverse is
$T=id+\nabla\psi$ with periodic $\psi$, normalized to have zero mean.
The Sobolev estimate for inversion near the identity then gives
$
\|\nabla\psi\|_{W^{k+1,p}}
\le C\|\nabla\phi\|_{W^{k+1,p}}
\le C'\|\rho_1-\rho\|_{W^{k,p}}.
$
\end{proof}

\subsection{Variation formulae} \label{appendix-variationformula}

Lemma~\ref{lem:first-second-var} supplies the Hessian formulas used in
Lemmas~\ref{lemma-convconc}--\ref{lemma-locconvconc} and the first
variation used in Proposition~\ref{prop-main}.

\begin{lemma}[Variation formulae, cf. Sec.~2 in \cite{CJS25}]\label{lem:first-second-var}\leavevmode\newline
Suppose $\mu=\rho\,dx$ is a probability measure on $\Tn$ such that
$\rho\in W^{1,p}(\Tn)$ for some $p>n$ and $0<c\le \rho \le C$ a.e. on $\Tn$.
Let $V\in C^2(\Tn)$ and define
\[
G(\mu)=\int_{\Tn}\rho\log\rho\,dx+\int_{\Tn} V\,d\mu .
\]
Then for any $\Phi=\nabla\phi$ and $\Psi=\nabla \psi$ in  $W^{2,p}(\Tn)$ we have the first variation formula
\[
\frac{d}{dt}\Big|_{t=0} G\bigl((id+t\Phi)_\#\mu\bigr)
=
-\int_{\Tn} \nabla\cdot\Phi\,d\mu
+
\int_{\Tn}\nabla V\cdot\Phi\,d\mu
=
\int_{\Tn}\bigl(\nabla\log\rho+\nabla V\bigr)\cdot\Phi\,d\mu ,
\]
and the second variation formula
\[ \begin{aligned} 
\frac{\partial^2}{\partial s \partial t }\Big|_{t=s=0} G\bigl((id+t\Phi+s\Psi )_\#\mu\bigr)
& =
\int_{\Tn}\Bigl(  \nabla\Phi: \nabla \Psi   +\nabla^2V[\Phi,\Psi ]\Bigr)\,d\mu \\
&= \int L\Phi\cdot \Psi \, d\mu = \int L\Psi \cdot \Phi \,d\mu ,
\end{aligned} \]
where $L\Phi= -\rho^{-1}\nabla\cdot(\rho\nabla\Phi)
  + \nabla^2 V\cdot\Phi.$
\end{lemma}

\begin{proof}
By \eqref{eq:analytic-embeddings}, $\Phi,\Psi\in C^{1,\alpha}$ for some
$\alpha>0$. If $|s|\|\nabla\Psi\|_\infty+|t|\|\nabla\Phi\|_\infty<1$,
then $T_{s,t}=\mathrm{id}+t\Phi+s\Psi$ is a diffeomorphism of the torus.
Change of variables gives the single formula
\[
G((T_{s,t})_\#\mu)
=\int\bigl[\log\rho-\log\det(I+t\nabla\Phi+s\nabla\Psi)
+V(T_{s,t})\bigr]\,d\mu.
\]
The first and mixed second derivatives of the integrand are uniformly
bounded near $(s,t)=(0,0)$, so differentiation under the integral is
valid. The identities
$D\log\det(I)[A]=\operatorname{tr}A$ and
$D^2(-\log\det)(I)[A,B]=\operatorname{tr}(AB)$ give the stated
variations; here $\nabla\Phi$ and $\nabla\Psi$ are symmetric, so
$\operatorname{tr}(\nabla\Phi\nabla\Psi)=\nabla\Phi:\nabla\Psi$.
Periodic integration by parts yields the first variation in gradient
form and the symmetric representation
$
\int  L\Phi \cdot \Psi \, d\mu 
=\int\bigl(\nabla\Phi:\nabla\Psi+\nabla^2V[\Phi,\Psi]\bigr)\,d\mu
= \int L \Psi \cdot \Phi \, d\mu $.
\end{proof}

The following identity at a Gibbs density gives the nonnegativity
used in the spectral-gap proof of Lemma~\ref{lemma-convconc}:
\begin{lemma}[Derivation of \eqref{eq-101}] \label{lemma-A7eq101} For a given probability density $\rho\propto e^{-V}$ with $V\in C^2$, there holds \[\int  (|\nabla^2  \phi|^2 + \nabla^2 V[\nabla\phi,\nabla\phi]) \, \rho  dx=\int |\Delta \phi -\langle \nabla\phi,\nabla V\rangle|^2 \, \rho  dx,\] for all $\nabla \phi \in W^{2,p}$, $p>n$. 

\begin{proof}
Set $\mathcal A=-\Delta+\nabla V\cdot\nabla$. Since
$\nabla\rho=-\rho\nabla V$, weighted integration by parts gives
\begin{align*}
\int |\mathcal A\phi|^2\rho\,dx
&=\int \nabla\phi\cdot\nabla(\mathcal A\phi)\,\rho\,dx=\sum_i\int \partial_i\phi\,\mathcal A(\partial_i\phi)\,\rho\,dx
  +\int \nabla^2V[\nabla\phi,\nabla\phi]\,\rho\,dx\\
&=\int\bigl(|\nabla^2\phi|^2+\nabla^2V[\nabla\phi,\nabla\phi]\bigr)\,\rho\,dx.
\end{align*}
This proves \eqref{eq-101} for smooth $\phi$; approximation gives the
stated regularity.
\end{proof}
\end{lemma}

\section{Supplementary proofs and discussion}\label{appendix-detailedproofs}

Appendix~\ref{App-CCstability} gives the estimate proving
Lemma~\ref{lemma-stability}, and Appendix~\ref{app-MFLDA} contains the
proofs of Lemmas~\ref{lemma-smoothing} and~\ref{lemma-interpolation}.
These support the proof of Theorem~\ref{thm-2}.
The remaining subsections discuss the spectral gap,
convergence rate, and role of entropy, followed by the NI comparison and
the compact-state oscillatory example from Section~\ref{sec-intro}.
The final remark discusses the sampling error in
Theorem~\ref{thm-particle}.

\subsection{Convex--concave stability of Wasserstein GDA}
\label{App-CCstability}

The following estimate proves Lemma~\ref{lemma-stability}, which
controls the initial regularization interval in the proof of
Theorem~\ref{thm-2}. Its EVI argument is the global counterpart of
the local contraction in Proposition~\ref{prop-main}.

\begin{lemma}[Convex--concave stability of Wasserstein GDA]\label{lemma-stability-cc}
Assume that for some $\lambda'\in\mathbb R$ the payoff $F(\mu,\nu)$ is
$\lambda'$-displacement convex in $\mu$ and $\lambda'$-displacement concave in $\nu$,
uniformly in the other variable.

Let $(\mu_t,\nu_t)_{t\ge0}$ be an absolutely continuous curve in
$\mathcal P_2(\mathbb T^n)\times\mathcal P_2(\mathbb T^n)$ solving the continuity
equations with velocity fields given by Wasserstein gradient descent--ascent:
\[
\partial_t \mu_t+\nabla\cdot(\mu_t v_t)=0,\qquad v_t=-\grad_{W_2,\mu}F(\mu_t,\nu_t),
\]
\[
\partial_t\nu_t+\nabla\cdot(\nu_t w_t)=0,\qquad w_t=+\grad_{W_2,\nu}F(\mu_t,\nu_t).
\]
Let $(\mu^*,\nu^*)$ be a stationary mixed Nash equilibrium, i.e. it satisfies the saddle condition
\[
F(\mu^*,\nu)\le F(\mu^*,\nu^*)\le F(\mu,\nu^*)\qquad \forall\,(\mu,\nu),
\]
and the stationarity conditions
\[
\grad_{W_2,\mu}F(\mu^*,\nu^*)=0,\qquad \grad_{W_2,\nu}F(\mu^*,\nu^*)=0.
\]
Then for all $t\ge 0$,
\[
W_2^2(\mu_t,\mu^*)+W_2^2(\nu_t,\nu^*)
\le e^{-2\lambda' t}\Bigl(W_2^2(\mu_0,\mu^*)+W_2^2(\nu_0,\nu^*)\Bigr).\]
\end{lemma}

\begin{proof}
Set $E_\mu(t)=W_2^2(\mu_t,\mu^*)$, $E_\nu(t)=W_2^2(\nu_t,\nu^*)$,
and $E=E_\mu+E_\nu$. Combining the above-tangent inequality
\citep[(10.1.7)]{AGS08} with Theorem~\ref{thm:first_variation_W2}
and the descent--ascent velocities gives, for a.e.\ $t$,
\begin{align}
\tfrac12 E_\mu'(t)
&\le F(\mu^*,\nu_t)-F(\mu_t,\nu_t)-\tfrac{\lambda'}2E_\mu(t),
\label{eq:EDI-mu}\\
\tfrac12 E_\nu'(t)
&\le F(\mu_t,\nu_t)-F(\mu_t,\nu^*)-\tfrac{\lambda'}2E_\nu(t).
\label{eq:EDI-nu}
\end{align}
Applying the same tangent inequality at the stationary equilibrium,
where both Wasserstein gradients vanish, yields
\[
\begin{aligned}
F(\mu^*,\nu_t)-F(\mu^*,\nu^*)&\le-\tfrac{\lambda'}2E_\nu(t),\\
F(\mu^*,\nu^*)-F(\mu_t,\nu^*)&\le-\tfrac{\lambda'}2E_\mu(t).
\end{aligned}
\]
Adding these bounds to \eqref{eq:EDI-mu}--\eqref{eq:EDI-nu} gives
\begin{equation}\label{eq:lyap-pre}
\tfrac12E'(t)
\le-\tfrac{\lambda'}2E(t)+F(\mu^*,\nu_t)-F(\mu_t,\nu^*)
\le-\lambda'E(t).
\end{equation}
Gr\"onwall's inequality proves the claim.
\end{proof}

\subsection{Smoothing and interpolation}\label{app-MFLDA}

We give the proofs of Lemmas~\ref{lemma-smoothing} and
\ref{lemma-interpolation}, which convert Wasserstein closeness into the
regularity needed for local contraction in the proof of
Theorem~\ref{thm-2}.

\noindent\textbf{Lemma~\ref{lemma-smoothing} (Smoothing)}
\emph{ For any $\mu_0$, $\nu_0$ in $\mathcal{P}(\Tn)$, $\delta>0$ and $\alpha\in(0,1)$, there is $C=C(f,\delta,\alpha)<\infty$ such that 
\[\Vert \mu_t\Vert_{C^{1,\alpha}}+\Vert \nu_t \Vert _{C^{1,\alpha}}\le C\quad \text{ for } t\ge \delta . \]}


\begin{proof}The proof proceeds as in \cite[Theorem 6.1]{CJS25}. We rewrite \eqref{eq-main} as \begin{equation}\label{eq-566}
\partial_t \mu_t - \Delta \mu_t = \nabla\cdot(\mu_t\nabla B_t),
\qquad
\partial_t \nu_t - \Delta \nu_t = \nabla\cdot(\nu_t\nabla C_t),
\end{equation}
where
\[
B_t(x)=\int_{\mathbb T^n} f(x,y)\,d\nu_t(y),
\qquad
C_t(y)=-\int_{\mathbb T^n} f(x,y)\,d\mu_t(x).
\]
Since $f\in C^2$, we have the uniform bound $\|B_t\|_{C^2},\|C_t\|_{C^2}\le \|f\|_{C^2}.$

\medskip
\noindent
\textbf{Uniform $L^2$ bound.}
Multiplying \eqref{eq-566} by $\mu_t$ and $\nu_t$, respectively, and integrating
by parts, 
\[
\partial_t \|\mu_t\|_{L^2}^2
\le -\|\nabla\mu_t\|_{L^2}^2 + C\|\mu_t\|_{L^2}^2,
\qquad
\partial_t \|\nu_t\|_{L^2}^2
\le -\|\nabla\nu_t\|_{L^2}^2 + C\|\nu_t\|_{L^2}^2,
\]
for some $C=C(f)$. By the Gagliardo--Nirenberg inequality \citep{nirenberg1959elliptic},
\[
\|\rho\|_{L^2}^{2+\frac{4}{n}}
\le C\Bigl(
\|\nabla\rho\|_{L^2}^2 \|\rho\|_{L^1}^{\frac{4}{n}}
+ \|\rho\|_{L^1}^{2+\frac{4}{n}}
\Bigr),
\]
which interpolates between the $L^1$ mass and the $H^1$ energy.
Since $\mu_t$ and $\nu_t$ are probability densities,
$\|\mu_t\|_{L^1}=\|\nu_t\|_{L^1}=1$, this yields $\|\nabla\mu_t\|_{L^2}^2
\ge \frac{1}{C}\|\mu_t\|_{L^2}^{2(1+\frac{2}{n})}-C$,
and hence
\[
\partial_t \|\mu_t\|_{L^2}^2
\le -\frac{1}{C}\|\mu_t\|_{L^2}^{2(1+\frac{2}{n})}
+ C(\|\mu_t\|_{L^2}^2+1).
\]
Integrating this differential inequality gives
\begin{equation} \label{eq-1626}
\|\mu_t\|_{L^2}^2 \le C(t^{-n/2}+1),
\end{equation}
and therefore a uniform $L^2$ bound for all $t\ge\delta$.
The same estimate holds for $\nu_t$.

\medskip
\noindent

\textbf{Regularization through bootstrapping.}
The remaining steps rely on standard regularity theory for linear parabolic equations.
We invoke the $L^p$-maximal regularity estimate (or Calder\'on--Zygmund estimate) for the heat equation
$\partial_t\rho_t-\Delta\rho_t=f_t$ \citep{lieberman1996second,Krylov2008}:
for $\delta\in(0,1)$, $k\in\mathbb Z$, $p,q\in(1,\infty)$, and any $s\ge 0$,
\begin{equation}\label{eq:max-reg-shifted}
\int_{s+\delta}^{s+1} \Bigl( \Vert \partial_t \rho_t \Vert _{W^{k,p}}^q +
\|\rho_t\|_{W^{k+2,p}}^q \Bigr)\, dt 
\le C_{\delta,k,p,q } \int_s^{s+1} \Bigl(\|f_t\|_{W^{k,p}}^q 
+ \|\rho_t\|_{W^{k,p}}^q 
\Bigr) \,dt .
\end{equation}
(Here the left-hand side is measured on the interior time interval $[s+\delta,s+1]$,
while the right-hand side uses $[s,s+1]$; the constant $C_{\delta,k,p,q}$ may blow up as $\delta\downarrow0$.)
Throughout this argument, we repeatedly apply \eqref{eq:max-reg-shifted} to \eqref{eq-566}
by viewing $\nabla \cdot (\mu_t \nabla B_t)$ and $\nabla \cdot ( \nu_t \nabla C_t)$ as the source term $f_t$.

First, by \eqref{eq-1626}, we may view
$\nabla \cdot (\mu_t \nabla B_t) \in L^\infty_t W_x^{-1,2}([s,s+1]) \subset L^{q}_t W^{-1,2}_x([s,s+1])$
for all $q<\infty$ for $s>0$.\footnote{
For any time interval $I\subset\mathbb R$, we use the shorthand
\[
L^q_t W^{k,p}_x(I):=L^q\bigl(I;W^{k,p}(\Tn)\bigr),\qquad
\|\rho\|_{L^q_t W^{k,p}_x(I)}:=\Bigl(\int_I \|\rho_t\|_{W^{k,p}(\Tn)}^q\,dt\Bigr)^{1/q}.
\]
If the interval $I$ is clear from the context, we omit it.
In particular, for any Banach space $X$ and any finite interval $I$,
$L^\infty(I;X)\hookrightarrow L^q(I;X)$ for every $q<\infty$.
For basic properties of Bochner spaces in time, see \cite[$\S$ 5.9.2]{evans2010partial}.
}
The estimate \eqref{eq:max-reg-shifted} with $k=-1$ and $p=2$ gives uniform
$L^q_tW^{1,2}_x([s+\delta,s+1])$ bounds of $\mu_t$ and $\nu_t$.
Using this improved estimate, applying \eqref{eq:max-reg-shifted} once more with $k=0$ and $p=2$
yields $L^q_tW^{2,2}_x([s+\delta,s+1])$ bounds of $\mu_t$ and $\nu_t$.
Although $\nabla B_t, \nabla^2 B_t \in L^\infty_t L^\infty_x \subset L^q_t L^\infty_x$ is not sufficiently regular to iterate in $k$,
the Sobolev embedding \eqref{eq:analytic-embeddings}
$W^{2,2}\hookrightarrow W^{1,2^*}$ with $2^*=2n/(n-2)$ when $n>2$
allows us to increase the spatial integrability exponent. For $n=2$ one
may take any finite target exponent, and for $n=1$ the embedding is into
$W^{1,\infty}$.
Iterating in $p$ while fixing $k=0$, and applying this argument on intervals of the form $[s,s+1]$,
we obtain that for every $p\in(1,\infty)$ and $q\in (1,\infty)$,
\begin{equation} \label{eq-1627}
 \int_{s}^{s+1}\Bigl(  \|\partial_t \mu_t\| _{L^p}^q + \|\mu_t\|^q _{W^{2,p}}
+\|\partial_t \nu_t\| _{L^p}^q+\|\nu_t\|^q _{W^{2,p}} \Bigr)\, dt 
\le C(f,\delta,p,q)
\quad\text{for all } s\ge \delta .
\end{equation} 

Next, differentiating the system \eqref{eq-main} in $t$, yields $\dot \mu_t = \partial_t \mu_t$ solves
\begin{equation}\label{eq-1629}
\partial_t \dot \mu_t -\Delta \dot \mu_t
= \nabla \cdot ( \dot \mu_t \nabla B_t+ \mu_t \nabla \dot B_t ),
\end{equation}
where $\dot B_t(x)=\int f(x,y)\,d\dot\nu_t(y)$.
Note $\nabla \dot B_t, \nabla^2 \dot B_t  \in L^q_t L^\infty_x$ for all $q\in(1,\infty)$:
indeed, for $|\beta|\le2$, $\| \nabla^\beta_x  \dot B_t\|_{L^\infty_x}
\le \|\nabla _x^\beta f\|_{L^\infty_{x,y}}\|\dot\nu_t\|_{L^1_y}
\le C \|\nabla _x^\beta f\|_{L^\infty_{x,y}}\|\dot\nu_t\|_{L^p_y}$, and this follows by \eqref{eq-1627}.
Using \eqref{eq-1627}, we repeat the previous bootstrapping argument for $\dot \mu_t$ and $\dot \nu_t$ on \eqref{eq-1629}.
More precisely, we first apply \eqref{eq:max-reg-shifted} with $k=-1$ to \eqref{eq-1629}
to obtain $\dot\mu_t,\dot\nu_t\in L^q_tW^{1,p}_x$ for all $p,q\in(1,\infty)$;
this in turn implies $\nabla\cdot(\dot\mu_t\nabla B_t)$ and $\nabla\cdot(\mu_t\nabla \dot B_t)\in L^q_tL^p_x$
for all $p,q\in(1,\infty)$, so we apply the second time with $k=0$.
As a consequence,
\begin{equation} \label{eq-16272}
 \int_{s}^{s+1} \Bigl( \|\partial_t \dot \mu_t\| _{L^p}^q + \|\dot \mu_t\|^q _{W^{2,p}}
+\|\partial_t \dot \nu_t\| _{L^p}^q+\|\dot \nu_t\|^q _{W^{2,p}} \Bigr) \, dt 
\le C(f,\delta,p,q)
\quad\text{for all } s\ge \delta .
\end{equation}

Finally, $\mu_t \in L^2_t W^{2,p}_x([s,1+s])$ and $\dot \mu_t \in L^2_t W^{2,p}_x([s,1+s])$ imply
$\mu_t \in C^{0,\frac12}_t W^{2,p}_x([s,1+s])$.
Namely $\mu_t \in W^{1,2}_t W^{2,p}_x([s,1+s]) \hookrightarrow C^{0,\frac12}_t W^{2,p}_x([s,1+s])$
and this in particular implies (e.g., see \cite[Theorem~2 in $\S$ 5.9.2]{evans2010partial})
\[
\Vert \mu _{t} \Vert_{W^{2,p}} \le C(f,\delta,p), \quad \text{ for all } t\ge \delta .
\]
The result follows by Morrey's embedding \eqref{eq:analytic-embeddings}
$W^{2,p}\hookrightarrow C^{1,\alpha}$ for $p$ sufficiently large.
\end{proof}

\bigskip 

\noindent\textbf{Lemma~\ref{lemma-interpolation} (Interpolation)}
\emph{ Let $0<\beta<\alpha<1$ and let $\rho_0\,dx,\rho_1\,dx \in \mathcal{P}(\Tn)$.
Assume $\|\rho_0\|_{C^{1,\alpha}}$, $\|\rho_1\|_{C^{1,\alpha}}\le C<\infty$.
Then for every $\varepsilon>0$ there exists $\delta=\delta(\varepsilon,C,\alpha,\beta)>0$
such that
\[
W_2(\rho_0,\rho_1)\le \delta
\quad\Longrightarrow\quad
\|\rho_0-\rho_1\|_{C^{1,\beta}}\le \varepsilon.
\]}
{\begin{proof}
The proof follows by compactness argument, as in \cite[Lemma~5.7]{CJS25}.
Assume the statement is false. Then there exist sequences $\rho_{0,n},\rho_{1,n}$
with $\|\rho_{i,n}\|_{C^{1,\alpha}}\le C$ for $i\in\{0,1\}$ and all $n$, such that
\begin{equation}\label{eq-tt1}
W_2(\rho_{0,n},\rho_{1,n})\le \frac1n,
\end{equation}
while $\|\rho_{0,n}-\rho_{1,n}\|_{C^{1,\beta}}\ge \varepsilon_0$ for some $\varepsilon_0>0$.
By the compact embedding $C^{1,\alpha}\hookrightarrow C^{1,\beta}$ in \eqref{eq:compact-embeddings}, passing to a subsequence
we may assume $\rho_{0,n}\to\rho_0$ and $\rho_{1,n}\to\rho_1$ in $C^{1,\beta}$. Recall that $W_2$-distance is controlled by $L^1$-distance on compact manifold: Let $D:=\mathrm{diam}(\Tn )<\infty$. By Kantorovich--Rubinstein duality,
\[
W_1(\rho_0,\rho_1)
=\sup_{\mathrm{Lip}(\varphi)\le 1}\int_{\Tn}\varphi(\rho_0-\rho_1)
\le D\|\rho_0-\rho_1\|_{L^1}.
\]
Moreover, for any coupling $\pi$ one has $\int d^2\,d\pi\le D\int d\,d\pi$, hence
$W_2^2(\rho_0,\rho_1)\le D\,W_1(\rho_0,\rho_1)$. Therefore $W_2(\rho_0,\rho_1)\le D\,\|\rho_0-\rho_1\|_{L^1}^{1/2}$.

By the triangle inequality and the observation above,
\[
W_2(\rho_0,\rho_1)
\le W_2(\rho_0,\rho_{0,n})+W_2(\rho_{0,n},\rho_{1,n})+W_2(\rho_{1,n},\rho_1)
\longrightarrow 0,
\]
hence $\rho_0=\rho_1$. Therefore,
\[
\|\rho_{0,n}-\rho_{1,n}\|_{C^{1,\beta}}
\le \|\rho_{0,n}-\rho_0\|_{C^{1,\beta}}+\|\rho_{1,n}-\rho_1\|_{C^{1,\beta}}
\longrightarrow 0,
\]
a contradiction.
\end{proof}
}
For the purposes of this paper, it suffices to use the soft compactness argument above, although one may obtain quantitative interpolation estimates, for example from \citep{peyre2018comparison}.

\subsection{Spectral gap, convergence rate, and entropy}
\label{app:spectral-discussion}

The following remarks give an alternative characterization of the
coercivity constant in Lemma~\ref{lemma-convconc}, discuss the rate in
Theorem~\ref{thm-2}, and clarify the role of entropic regularization.

\begin{remark}[Poincar\'e inequality and an alternative route]
\label{remark-alternative}
Once \eqref{eq-secondvariation} and \eqref{eq-101} are established, the remaining step in the proof of Lemma~\ref{lemma-convconc} may also be phrased in terms of the spectral gap of an associated weighted scalar Laplacian. More precisely, let
\[
\mathcal A_{\mu^*}\phi
:=
-\Delta \phi+\langle \nabla \phi,\nabla V_{\nu^*}\rangle
\]
be the nonnegative self-adjoint operator on \(L^2(\mu^*)\), where
\(d\mu^*=Z^{-1}e^{-V_{\nu^*}}\,dx\). The proof shows that \eqref{eq-conv1} in Lemma~\ref{lemma-convconc} follows by establishing the estimate
\begin{equation}\label{eq-equiv11}
\int |\nabla \phi|^2\,d\mu^*
\;\le\;
\lambda_{\rm vec}^{-1}
\int |\mathcal A_{\mu^*}\phi|^2\,d\mu^*
\end{equation}
for some \(\lambda_{\rm vec}=\lambda_{\rm vec}(f)>0\). This gap should be distinguished from the scalar Poincar\'e gap, namely a constant
\(\lambda_{\rm sc}>0\) such that
\begin{equation}\label{eq-equiv12}
\int \phi^2\,d\mu^*
-
\Bigl(\int \phi\,d\mu^*\Bigr)^2
\;\le\;
\lambda_{\rm sc}^{-1}
\int |\nabla \phi|^2\,d\mu^* .
\end{equation}
By the spectral theorem, or equivalently by the heat-flow argument in
\cite[Prop.~4.8.3]{bakry2013analysis}, these two gaps are equivalent: \eqref{eq-equiv11} holds for a given \(\lambda>0\) if and only if \eqref{eq-equiv12} holds for the same \(\lambda\).

The existence of \(\lambda_{\rm sc}>0\) in \eqref{eq-equiv12} is a classical consequence of the spectral theory of weighted scalar Laplacians. Indeed, the largest possible \(\lambda_{\rm sc}\) is the first nonzero eigenvalue, or equivalently the spectral gap, of \(\mathcal A_{\mu^*}\) on \(L^2_0(\mu^*)\)\footnote{The operator \(\mathcal A_{\mu^*}\) always has \(0\) as an eigenvalue, with constant eigenfunctions. Here the (positive) spectral gap refers to the first nonzero eigenvalue of \(\mathcal A_{\mu^*}\), equivalently the spectral gap of its restriction to the zero-mean subspace \(L^2_0(\mu^*)\).}. The Holley--Stroock perturbation theorem \cite[Prop.~4.2.7]{bakry2013analysis} gives a way to estimate this scalar gap from below: it applies to the weighted measure
\(\mu^*=Z^{-1}e^{-V_{\nu^*}}dx\) as a bounded perturbation of Lebesgue measure on
\(\mathbb T^n\). Since the Laplacian \(-\Delta\) on \(\mathbb T^n\) has spectral gap
\(4\pi^2\), the Holley--Stroock yields
\[
\lambda_{\rm sc}
\ge
4\pi^2 e^{-\operatorname{osc} V_{\nu^*}}
\ge
4\pi^2 e^{-\operatorname{osc} f}
>0, \quad \text{where } \operatorname{osc} f=\sup f - \inf f \ge 0.
 \]
By the equivalence above, the same lower bound gives \eqref{eq-equiv11}. Applying the same argument to the \(\nu\)-variable, and taking the minimum of the two constants, one obtains Lemma~\ref{lemma-convconc}; in particular, one may take a positive constant bounded below by \(4\pi^2 e^{-\operatorname{osc} f}\).

We present the argument from \eqref{eq-lambdagap} to
\eqref{eq-energyargument} to keep the proof self-contained. Once the required identities and estimates are established, the remaining spectral step is just the transparent functional-analytic compactness argument presented in the proof of the lemma.
\end{remark}

\bigskip

\begin{remark}[Rate of convergence]\label{remark-sharpasymp}
According to the proofs of Theorem~\ref{thm-2} and Lemma~\ref{lemma-convconc} in Section~\ref{sec-MFLDA}, if $(\mu_t,\nu_t)$ converges to the mixed Nash equilibrium (MNE) $(\mu^*,\nu^*)$, then
\begin{equation}\label{eq-gapp}
W_2(\mu_t,\mu^*) + W_2(\nu_t,\nu^*) = O(e^{-\lambda t})
\quad \text{for all } \lambda \in (0,\lambda_{gap}).
\end{equation}
Moreover, from the proof of Lemma~\ref{lemma-convconc}, the constant $\lambda_{gap}(f)>0$  can be characterized as
\[
\lambda_{gap} = \min\{\lambda_{gap}^{\mathcal X},\,\lambda_{gap}^{\mathcal Y}\},\text{ where}
\]
\[
\lambda_{gap}^{\mathcal X}
=
\inf_{\substack{\phi \in W^{2,2}(\mathcal X)\\ \nabla\phi \neq 0}}
\frac{\int \bigl(|\nabla^2\phi|^2 + \langle \nabla\phi,\nabla^2 V_{\nu^*}\nabla\phi\rangle\bigr)\,d\mu^*}
{\int |\nabla\phi|^2\,d\mu^*},
\qquad
V_{\nu^*}(x)=\int f(x,y)\,d\nu^*(y),
\]
\[
\lambda_{gap}^{\mathcal Y}
=
\inf_{\substack{\phi \in W^{2,2}(\mathcal Y)\\ \nabla\phi \neq 0}}
\frac{\int \bigl(|\nabla^2\phi|^2 - \langle \nabla\phi,\nabla^2 U_{\mu^*}\nabla\phi\rangle\bigr)\,d\nu^*}
{\int |\nabla\phi|^2\,d\nu^*},
\qquad
U_{\mu^*}(y)=\int f(x,y)\,d\mu^*(x).
\]
Note that, by Remark \ref{remark-alternative}, we have alternative characterizations of $\lambda^{\mathcal{X}}_{gap}$ and $\lambda^{\mathcal{Y}}_{gap}$ and bounds:
\[\lambda_{gap}^{\mathcal{X}}= \inf _{\substack{\phi \in W^{1,2}(\mathcal X)\\ \phi \neq 0,\int \phi d\mu^*=0}} \frac{\int |\nabla \phi|^2 d\mu^* }{\int \phi^2 d\mu^*}\ge \frac{\sigma_*(\mathcal{X})}{e^{\operatorname{osc}f}}, \quad \lambda_{gap}^{\mathcal{Y}}= \inf _{\substack{\phi \in W^{1,2}(\mathcal Y)\\ \phi \neq 0,\int \phi d\nu^*=0}} \frac{\int |\nabla \phi|^2 d\nu^* }{\int \phi^2 d\nu^*} \ge \frac{\sigma_*(\mathcal{Y})}{e^{\operatorname{osc}f}}.\]
Here $\operatorname{osc}f:= \sup f-\inf f\ge 0$ and $\sigma_*(\mathcal{X})>0$ denotes the first nonzero spectral value of unweighted Laplacian on $\mathcal{X}$, which is $4\pi^2$ for $\mathcal{X}=\Tn$.
\end{remark}

\begin{remark}[Entropy alone does not ensure stability]
\label{rem:entropy-not-automatic}
The coercive spectral gap in Lemma~\ref{lemma-convconc} and the local stability
in Theorem~\ref{thm-2} should not be interpreted as a generic consequence of
adding entropy. On the flat torus, the entropy functional \(H\) is
\(0\)-displacement convex, but not uniformly displacement convex with a positive
constant; more generally, the displacement convexity constant of entropy is tied
to the Ricci curvature lower bound of the underlying space
\cite{mccann1997convexity,villani2008optimal}. Thus, adding \(H\)
to a functional does not by itself shift the Wasserstein Hessian by a positive number.

Related phenomena already occur for single-player mean-field Langevin dynamics.
For free energies with nonlinear interaction terms, entropy contributes
diffusion and smoothing, but the Wasserstein Hessian at a stationary state may
still have negative spectrum, in which case the stationary state is unstable;
alternatively, the bottom of the relevant spectrum may be zero, so that no
coercive gap is available. For instance, \cite{MR4983182} study the attractive
log gas, equivalently in dimension two a periodic Patlak--Keller--Segel model,
and identify a sharp temperature threshold at which the uniform stationary
state undergoes a spectral/stability transition. See also
\cite{Monmarche:2025aa} for related phenomena in the Curie--Weiss model.

In the present two-player problem, the positivity used in Lemma~\ref{lemma-convconc} comes from the mixed Nash equilibrium structure. More precisely, after using the equilibrium relation
\(d\mu^*\propto e^{-V_{\nu^*}}dx\), the second variation can be rewritten as the square in \eqref{eq-101}. This gives nonnegativity at the equilibrium, and the compactness argument, equivalently the associated Poincar\'e gap, then upgrades this nonnegativity to strict coercivity; this is where the entropy-induced ellipticity enters. Thus, the mechanism is the combination of the equilibrium structure and the entropy-induced elliptic gap, not entropic regularization alone.
\end{remark}

\subsection{Comparison between notions of convergence}
\label{appendix-strong-NI}\label{app:comparison-notions}
The following comparison explains the convergence in NI error
deduced from Theorem~\ref{thm-2} in Section~\ref{sec-intro}.
For the functional $F$ in \eqref{eq-functional}, the Nikaido--Isoda error is
\begin{equation}\label{eq:def-NI}
\mathrm{NI}(\mu,\nu)
:=\max_{\nu'\in\mathcal P(\Tn)}F(\mu,\nu')
-\min_{\mu'\in\mathcal P(\Tn)}F(\mu',\nu).
\end{equation}
As observed in \citet[Eq.~(4)]{WC24_MFLDA}, for some
$\rho^*>0$ depending on $f$,
\begin{equation}\label{eq:W2-KL-NI}
\frac{\rho^*}{2}\bigl(W_2^2(\mu,\mu^*)+W_2^2(\nu,\nu^*)\bigr)
\le\mathrm{KL}(\mu\|\mu^*)+\mathrm{KL}(\nu\|\nu^*)
\le\mathrm{NI}(\mu,\nu).
\end{equation}
The converse implication needed here uses the convergence of densities.

\begin{lemma}[$C^0$ controls $\mathrm{NI}$]\label{lem:C1alpha-implies-NI}
Let $\mu_i=\rho_i\,dx$ and $\nu_i=\eta_i\,dy$, $i=1,2$, with
$C_1^{-1}\le\rho_i,\eta_i\le C_1$ on $\Tn$ for some $C_1\ge1$.
Then
\[
|\mathrm{NI}(\mu_1,\nu_1)-\mathrm{NI}(\mu_2,\nu_2)|
\le C\bigl(\|\rho_1-\rho_2\|_\infty+\|\eta_1-\eta_2\|_\infty\bigr),
\]
where $C=C(C_1,\|f\|_\infty)$. In particular, convergence of
$(\mu_t,\nu_t)$ to $(\mu^*,\nu^*)$ in $C^0$ implies
$\mathrm{NI}(\mu_t,\nu_t)\to0$ by \eqref{eq-C^2bound}.
\end{lemma}

\begin{proof}
Set $U_\mu(y)=\int f(x,y)\,d\mu(x)$ and
$V_\nu(x)=\int f(x,y)\,d\nu(y)$. The Gibbs variational formula gives
\begin{equation}\label{eq:NI-closed-form}
\mathrm{NI}(\mu,\nu)
=H(\mu)+H(\nu)+\log\int e^{U_\mu}\,dy+\log\int e^{-V_\nu}\,dx.
\end{equation}
On $[C_1^{-1},C_1]$, $r\mapsto r\log r$ is
$(1+\log C_1)$-Lipschitz. Moreover,
$|\log\int e^a-\log\int e^b|\le\|a-b\|_\infty$ and
$\|U_{\mu_1}-U_{\mu_2}\|_\infty
\le\|f\|_\infty\|\rho_1-\rho_2\|_\infty$, with the analogous
estimate for $V_\nu$. Substitution into \eqref{eq:NI-closed-form}
proves the claim.
\end{proof}

\subsection{Periodization of the oscillatory example}
\label{app:periodic-counterexample}

We adapt
\citet[Theorem~3.3]{MourratPillaudVivien26Oscillating} to the torus.
Their payoff $f_\alpha(u,v)=\alpha uv+g(u)-g(v)$,
where $g(z)=-z^2/2+z^4/12$, gives the deterministic system
\begin{equation}\label{eq:oscillatory-ode}
\dot m=b_1(m,n):=-\alpha n+m-\frac{m^3}{3},\qquad
\dot n=b_2(m,n):=\alpha m+n-\frac{n^3}{3}.
\end{equation}
For large $\alpha$, their Propositions~2.1, 2.2, and~2.4 give an
attracting cycle $\mathscr C_\alpha\subset\{3\le m^2+n^2\le6\}$,
its cycle averages, and a Lyapunov estimate near it.
On $\mathbb T_L=\mathbb R/(2\pi L\mathbb Z)$, set
\[
s_L(x)=L\sin(x/L),\qquad
f_{\alpha,L}(x,y)=f_\alpha(s_L(x),s_L(y)),\qquad
a_L(z)=1-z^2/L^2.
\]
For the torus processes with diffusion $\varepsilon\in(0,1]$, write
$Z_t=s_L(X_t)$, $W_t=s_L(Y_t)$ and
$(m_t,n_t)=(\mathbb E Z_t,\mathbb E W_t)$. It\^o's formula gives
\begin{equation}\label{eq:periodic-projected-sde}
dZ_t=\left[a_L(Z_t)b_1(Z_t,n_t)-\frac{\varepsilon Z_t}{L^2}\right]dt
      +\sqrt{2\varepsilon a_L(Z_t)}\,d\widetilde B_t,
\end{equation}
The analogous equation for $W_t$ uses $b_2(m_t,W_t)$ and independent
Brownian noise. Start from Dirac data with projected means on
$\mathscr C_\alpha$. Fix a bounded open neighborhood $U$ of the cycle
and set $T_*:=\inf\{t\ge0:(m_t,n_t)\notin U\}$.
All estimates below are initially restricted to $t<T_*$.

With $n_t$ fixed, the generator of the projected process $Z_t$ is
\[
\mathcal L_{L,t}h(z)
=\left[a_L(z)b_1(z,n_t)-\frac{\varepsilon z}{L^2}\right]h'(z)
 +\varepsilon a_L(z)h''(z).
\]
For $h(z)=e^{z^2}$, $\frac{\mathcal L_{L,t}h(z)}{h(z)}
=a_L(z)\left[-2\alpha n_tz+(2+4\varepsilon)z^2
                 -\frac23z^4+2\varepsilon\right]
 -\frac{2\varepsilon z^2}{L^2}$.
For $t<T_*$, boundedness of $U$ gives $|n_t|\le M$.
The negative quartic term therefore makes the bracket
at most $-2\varepsilon$ outside a fixed spatial interval $|z|\le R$,
independent of $t,L$. Hence $a_L(z)+z^2/L^2=1$ yields
$\mathcal L_{L,t}h\le C_{\alpha,\varepsilon}-\varepsilon h$.
The identity $\frac{d}{dt}\mathbb E h(Z_t)
=\mathbb E[\mathcal L_{L,t}h(Z_t)]$ then bounds all fixed-order
moments uniformly in $t<T_*$ and $L$, also for $W_t$.

These bounds control the modifications to their
Propositions~3.1--3.2. Put $\xi_t=Z_t-m_t$,
$r_t(z)=\alpha n_tz^2-z^3+z^5/3-\varepsilon z$.
Their equations~(3.3) and~(3.5) acquire the respective additions
\[
\frac{\mathbb E r_t(Z_t)}{L^2}\quad\hbox{to }\dot m_t,
\qquad
\frac{\mathbb E[\xi_t r_t(Z_t)]-\varepsilon\mathbb E Z_t^2}{L^2}
\quad\hbox{to }\frac12\frac{d}{dt}\mathbb E\xi_t^2.
\]
These and the analogous higher-moment remainders are uniformly
$O_{\alpha,\varepsilon}(L^{-2})$. For $L\ge L_0(\alpha,\varepsilon)$,
make the remainders in the second-moment sum and the finitely many
required centered-moment equations at most one in absolute value.
Proposition~3.1's calculation then gives
$\tfrac12\dot u_t\le u_t-u_t^2/6+2\varepsilon+1$ for
$u_t=\mathbb E(Z_t^2+W_t^2)$; Proposition~3.2's dissipative estimates
give the required centered-moment bounds, independent of
$\alpha$, $\varepsilon\le1$, $L$ under this condition.

The within-period estimates of
\citet[proof of Theorem~3.3]{MourratPillaudVivien26Oscillating}
keep the means near the cycle and the variances small.
Their end-of-period variance estimate restores the induction bounds,
using the cycle averages in their Proposition~2.2 and slow variation
of the third centered moment.
Choose $\alpha$ large, then $\varepsilon$ small as there, and finally
$L\ge L_0(\alpha,\varepsilon)$ large enough to absorb the remainders
into their strict margins. Renew the induction on completed periods
and apply the within-period estimates up to $T_*$: the means stay in a smaller neighborhood
compactly contained in $U$, so continuity excludes $T_*<\infty$.
The induction thus keeps the means rotating for all time.
Since $s_L$ is bounded and continuous, this precludes
convergence of the torus laws. Finally,
$(\theta,\varphi)=(x,y)/(2\pi L)$ and
$\tau=\varepsilon t/(2\pi L)^2$ give unit-diffusion MFL--DA on
$\mathcal X=\mathcal Y=\mathbb T$, with payoff $\widehat f_{\alpha,L,\varepsilon}(\theta,\varphi)
=\varepsilon^{-1}f_{\alpha,L}(2\pi L\theta,2\pi L\varphi).$

\subsection{Sampling error in the particle estimate}
\label{app:particle-sampling}

\begin{remark}[Sampling rate]
\label{rem:particle-sampling-rate}
The empirical-measure estimate of
\citet[Theorem~1]{fournier2015rate}, used in
Proposition~\ref{prop:stopped_expectations}, allows the choices
\[
\eta=\frac14 \quad (n\le3),\qquad
0<\eta<\frac14 \quad (n=4),\qquad
\eta=\frac1n \quad (n>4)
\]
in Theorem~\ref{thm-particle}. The strict inequality at $n=4$
comes from the bound
$O(N^{-1/4}\sqrt{\log(1+N)})$ before replacing it by a pure power. Since the MNE marginals have strictly positive continuous densities,
the quantization bound of
\citet[Theorem~1.2]{kloeckner2012approximation} gives
\[
\inf_{\substack{\rho_N\in\mathcal P(\mathbb T^n)\\
                 \#\operatorname{supp}\rho_N\le N}}
W_2(\rho_N,\mu^*)\ge c_{\mu^*}N^{-1/n},
\]
and similarly for $\nu^*$. In particular, the exponent $\eta$ is optimal
for $n>4$. 
\end{remark}

\end{document}